\documentclass[letterpaper, 10 pt, journal, twoside]{ieeetran}  % Comment this line out if you need a4paper
\IEEEoverridecommandlockouts                              % This command is only needed if 
\usepackage{bm}				% bm fonts
\usepackage{bbm}	
\usepackage{amsthm}
\usepackage{amsfonts}
\usepackage{amsmath}
\usepackage{mathtools}      % Loads amsmath
\usepackage{floatrow}
\usepackage{mathptmx}       % selects Times Roman as basic font
\usepackage{helvet}         % selects Helvetica as sans-serif font
\usepackage{courier}        % selects Courier as typewriter font
\usepackage{type1cm}        % activate if the above 3 fonts are
\usepackage{makeidx}         % allows index generation
\usepackage{graphicx}        % standard LaTeX graphics tool
\usepackage{multicol}        % used for the two-column index
\usepackage{multirow}
\usepackage[bottom]{footmisc}% places footnotes at page bottom
\usepackage{array}
\usepackage{textcomp}
\usepackage{comment}
\usepackage{cite}
\usepackage{wrapfig}
\usepackage[rightcaption]{sidecap}
\usepackage{mathrsfs}
\usepackage{subfigure}
\usepackage[colorinlistoftodos]{todonotes}

\usepackage{algorithmic} %format of the algorithm}
\usepackage[linesnumbered,ruled, vlined,commentsnumbered]{algorithm2e}
\SetKwComment{Comment}{}{}
\DeclareMathAlphabet\mathbfcal{OMS}{cmsy}{b}{n}
\DeclareMathAlphabet\mathcal{OMS}{cmsy}{m}{n}
\DeclareFontFamily{OT1}{mathc}{}
\DeclareFontShape{OT1}{mathc}{m}{n}{ <-> mathc10 }{}

\newtheorem{Pro}{Problem}

\newtheorem{Lem}{Lemma}
\newtheorem{Cor}{Corollary}

\newtheorem{remark}{Remark}

\newcommand{\argmin}[1]{\underset{#1}
	{\operatorname{arg}\,\operatorname{min}}\;}

\makeindex             % used for the subject index

\title{\LARGE Graph-based Proprioceptive Localization Using a Discrete Heading-Length Feature Sequence Matching Approach} 

\author{Hsin-Min Cheng and Dezhen Song 
\thanks{H. Cheng and D. Song are with CSE Department, Texas A\&M University, College Station, TX 77843, USA, Emails: \texttt{hmcheng@tamu.edu} and \texttt{dzsong@cs.tamu.edu}.}
\thanks{This work was supported in part by National Science Foundation under NRI-1748161 and NRI-1925037.}
}

\begin{document}
\maketitle
%\thispagestyle{empty}
%\pagestyle{empty}
%%%%%%%%%%%%%%%%%%%%%%%%%%%%%%%%%%%%%%%%%%%%%%%%%%%%%%%%%%%%%%%%%%%%%%%%%%%%%%%%
\begin{abstract}
Proprioceptive localization refers to a new class of robot egocentric localization methods that do not rely on the perception and recognition of external landmarks. These methods are naturally immune to bad weather, poor lighting conditions, or other extreme environmental conditions that may hinder exteroceptive sensors such as a camera or a laser ranger finder. These methods depend on proprioceptive sensors such as inertial measurement units (IMUs) and/or wheel encoders. Assisted by magnetoreception, the sensors can provide a rudimentary estimation of vehicle trajectory which is used to query a prior known map to obtain location. Named as graph-based proprioceptive localization (GBPL), we provide a low cost fallback solution for localization under challenging environmental conditions. As a robot/vehicle travels, we extract a sequence of heading-length values for straight segments from the trajectory and match the sequence with a pre-processed heading-length graph (HLG) abstracted from the prior known map to localize the robot under a graph-matching approach. Using the information from HLG, our location alignment and verification module compensates for trajectory drift, wheel slip, or tire inflation level. We have implemented our algorithm and tested it in both simulated and physical experiments. The algorithm runs successfully in finding robot location continuously and achieves localization accurate at the level that the prior map allows (less than 10m). 
\end{abstract}

%%%%%%%%%%%%%%%%%%%%%%%%%%%%%%%%%%%%%%%%%%%%%%%%%%%%%%%%%%%%%%%%%%%%%%%%%%%%%%%%
\section{Introduction}

Localization is a critical navigation function for vehicles or robots in urban area. Common localization methods employ global position system (GPS), a laser ranger finder, and a camera which are exteroceptive sensors relying on the perception and recognition of landmarks in the environment. However, high-rise buildings may block GPS signals. Poor weather and lighting conditions may challenge all exteroceptive sensors. What is needed is a fallback solution that enables vehicles to localize themselves under challenging conditions. This complements existing exteroceptive sensor-based localization methods. Inspired by biological systems, we combine proprioceptive sensors, such as inertial measurement units (IMU) and wheel encoders, with magnetoreception, to develop a map-based localization method to address the problem, which is named as graph-based proprioceptive localization (GBPL).

In a nutshell, our new GBPL method employs the proprioceptive sensors to estimate vehicle trajectory and match it with a prior known map. However, this is non-trivial because 1) there is a significant drift issue in the dead reckoning process and 2) the true vehicle trajectory does not necessarily match the street GPS waypoints on the map due to the fact that a street may contain multiple lanes and street GPS waypoints may be inaccurate. This determines that a simple trajectory matching would not work. Instead, we focus on matching features which are straight segments of the trajectory (Fig.~\ref{fig:system}). We keep track of connectivity, heading and length of each segment which converts the trajectory to a discrete and connected query sequence. This allows us to formulate the GBPL problem as a probabilistic graph matching problem.  To facilitate the Bayesian graph matching, we pre-process the prior known map consisting of GPS waypoints into a heading-length graph (HLG) to capture the connectivity of straight segments and their corresponding heading and length information. As the robot travels, we perform sequential Bayesian probability estimation until it converges to a unique solution. With global location obtained, we track robot locations continuously and align the trajectory with HLG to bound error drift.

\begin{figure}[t!]
	\center
	\includegraphics[width=1\linewidth, viewport=10 400 860 761, clip=true]{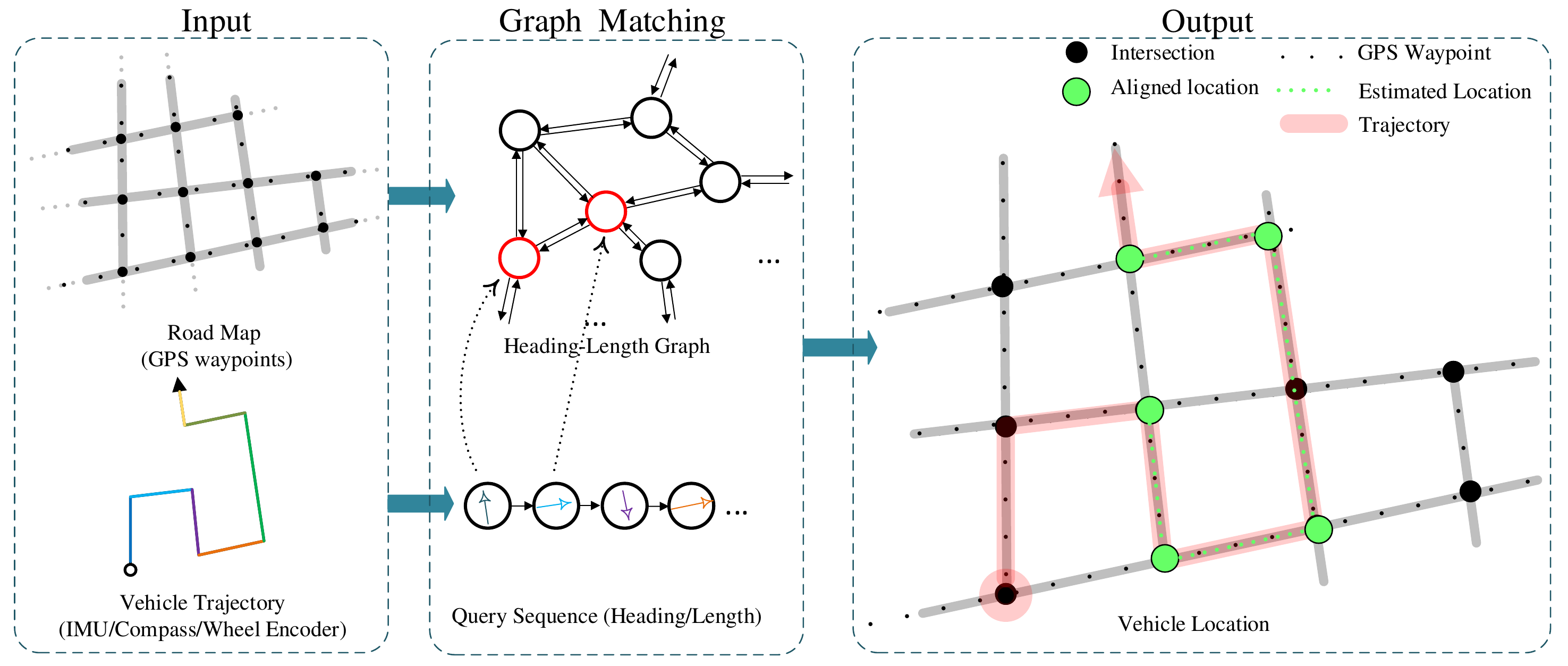}
	\caption{An illustration of GBPL method. Left: our inputs include a prior known map and the trajectory estimated from an IMU, a compass, and a wheel encoder. Middle: we process the prior map in to a straight segment connectivity graph and also the trajecory into a query sequence of headings and lengths of straight segments. Right: Aligned trajectory to the map after graph matching.}
	\label{fig:system}
\end{figure}

We have implemented our algorithm and tested it in physical experiments using our own collected data and an open dataset. The algorithm successfully and continuously localizes the robot. The experimental results show that our method outperforms in localization speed and robustness when compared with the counterpart in~\cite{cheng2018plam}. The algorithm achieves localization accurate at the level that the prior map allows (less than 10m). 

The rest of the paper is organized as follows. 
After a review of related work in Section~\ref{sec:relatedwork}, we define the problem in Section~\ref{sec:pro_def}. We introduce overall system design and detail GBPL in Section~\ref{sec:plam_model}. We validate our system and algorithm with simulation and physical experiments in Section~\ref{sec:exps} and conclude the paper in Section~\ref{sec:conclusion}.
 
%%%%%%%%%%%%%%%%%%%%%%%%%%%%%%%%%%%%%%%%%%%%%%%%%%%%%%%%%%%%%%%%%%%%%%%%%%%%%%%%
\section{Related Work} \label{sec:relatedwork}
Our GBPL is related to localization using different sensor modalities, dead-reckoning, and map-based localization. 

We can classify the localization methods into two categories based on sensor modalities: exteroceptive sensors and proprioceptive sensors. Exteroceptive sensors mainly rely on the perception and recognition of landmarks in the environment to estimate location. Mainstream exteroceptive sensors include cameras~\cite{Brubaker16selfloc,lowry2016visual,yan-mfg-lba-tro-2015} and laser range finders~\cite{hahnel2003efficient,hess2016real,levinson2010robust}. These methods are often challenged by poor lighting conditions or weather conditions. GPS receiver~\cite{hunter2014path,cui03gpsloc} is another commonly-used sensor but it malfunctions when the vehicle travels close to high-rise buildings or inside tunnels. On the other hand, proprioceptive sensors, such as IMUs~\cite{Aly17DrAnchor} and wheel encoders~\cite{chou-icra-2018}, are inherently immune to external conditions. However, they are more susceptible to error drift and suffer from limited accuracy. Recent sensor fusion approaches that combine an exteroceptive sensor, such as a camera or a laser ranger finder, with a proprioceptive sensor such as an IMU, greatly improve system robustness and become popular in applications~\cite{li2013high}. However, the sensor fusion approaches still strongly depend on exteroceptive sensor and cannot handle the aforementioned challenging conditions. 

% dead reackoning
To utilize proprioceptive sensors for navigation, dead reckoning integrates sensor measurements to compute robot/vehicle trajectory. The sensor measurements often include readings from accelerometers, gyroscopes, and/or wheel encoders~\cite{yi2007imu}.  
There are many applications using the dead reckoning approach such as autonomous underwater vehicles (AUVs)~\cite{paull2013auv} and pedestrian step measurement~\cite{kang2014smartpdr,constandache2010compacc}. To estimate the state of the robot/vehicle, filtering-based schemes such as unscented Kalman filter (UKF)~\cite{karras2010line} and particle filter (PF)~\cite{gustafsson2002particle,huang2010autonomous} are frequently employed. However, the nature of dead reckoning causes it to inevitably accumulate errors over time and lead to significant drift. To reduce the error drift, different methods have been proposed such as applying velocity constraint on wheeled robots~\cite{siciliano2016springer} and modeling the wheel slip for skid-steered mobile robots~\cite{yi2007imu}. These approaches have reduced error drift but cannot remove it completely. Error still accumulates over time and causes localization failure. To fix the issue, we will show that drift can be bounded to map accuracy level by using map matching if the filtering-based approach with graph matching are combined.

% map type, map-based localization methods
Our method is a map-based localization~\cite{Brubaker16selfloc,merriaux2015fast,ruchti2015localization,jiang2017geometric,jin2016robust}.  
According to~\cite{Thrun2005}, map representation can be classified into two categories: the location-based and the feature-based. The location-based maps are represented with specific locations of objects. For example, those existing geographic maps consisted of coordinate of locations such as OpenStreetMaps\texttrademark~(OSM)~\cite{OpenStreetMap} and Google Maps~\cite{GoogleMaps}. Geographic maps have been widely used to improve upon GPS measurements and there are common measures being used such as point-to-point, point-to-curve, curve-to-curve matching or advanced techniques\cite{quddus2015GPSmap}. The feature-based map is consisted with features of interest with its location. An example is ORB features~\cite{rublee2011orb} for visual simultaneous localization and mapping. In this work, we extract heading-length graph from geographic maps which converts a location-based map to a feature-based map to facilitate robust localization which also reduces graph size to speed up computation in the process. 

Closely-related works include~\cite{merriaux2015fast,Wahlstrom2016MapDR,yu2019hybrid}, which focus on map-aided localization using proprioceptive sensors for mobile robots. In \cite{Wahlstrom2016MapDR}, only vehicle speed and speed limit information from map are used as a minimal sensor setup. However, known initial position is required and the method achieves an accuracy of around 100 meters. In~\cite{merriaux2015fast},  the velocity from wheel encoder and steering angles are used for odometry and a particle filter based map matching scheme helps estimate vehicle positions. It does not consider velocity errors from the wheel encoder such as slippery or inflation levels. In~\cite{yu2019hybrid}, odometer and gyroscope readings are used for extended Kalman filter (EKF)-based dead reckoning and a map is used to correct errors when driving a long distance or turning at road intersections. The average positional error is 5.2 meters, but it again requires an initial position from GPS. It is worth noting that our localization solution does not require a known initial position. 

This paper is a significant improvement over our early work~\cite{cheng2018plam} where only heading sequence is used and localization is only intermittent for turns. The new method enables continuous localization by considering wheel encoder inputs and is less limited by map degeneracy (e.g. rectilinear environments). Also, we bound error drift in location alignment and verification after graph matching.

%%%%%%%%%%%%%%%%%%%%%%%%%%%%%%%%%%%%%%%%%%%%%%%%%%%%%%%%%%%%%%%%%%%%%%%%%%%%%%%%
\section{Problem Formulation}\label{sec:pro_def} 
\begin{figure*}[ht!]
	\includegraphics[width=1\linewidth, viewport=55 104 1715  623, clip=true]{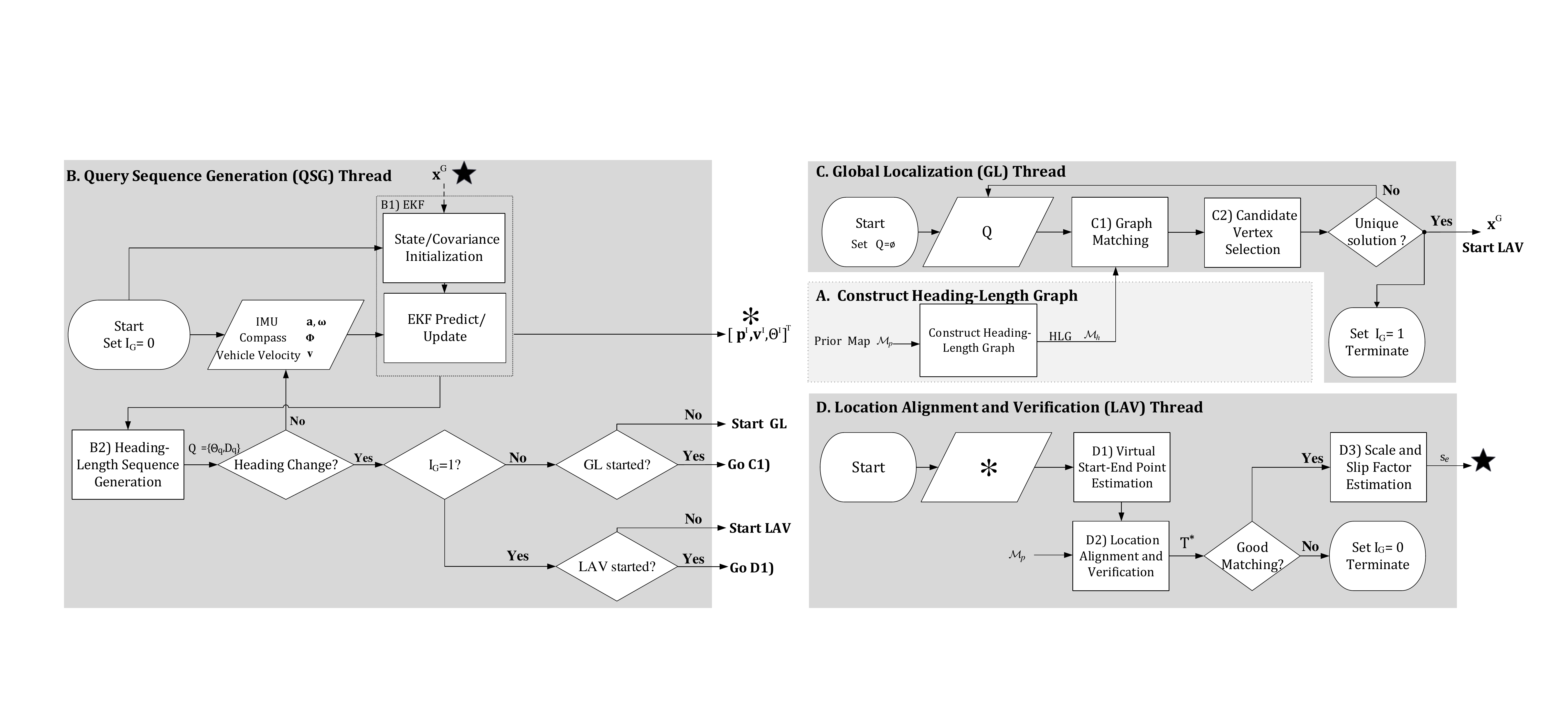}\label{fig:system-ET}
	\caption{System Diagram}
	\label{fig:SystemDiagram}
\end{figure*}

In our set up, a robot or a vehicle (We interchangeably use ``robot" and ``vehicle.'') is navigating in a poor weather conditions such as a severe thunderstorm or a whiteout snowstorm. No other exteroceptive sensors work properly. However, it is still necessary for the vehicle to find its location. 

The vehicle/robot is equipped with an IMU, a digital compass or a magnetometer, and an on-board diagnostics (OBD) scanner which provides velocity feedback while navigating in an area with a given prior road map, e.g. OpenStreetMaps (OSM)~\cite{OpenStreetMap}.  We have the following assumptions:
 
\begin{enumerate}
    \item[a.0] If needed, the vehicle is willing to change its course by making additional turns to assist our algorithm to find its location.
	\item[a.1]
	The robot is a nonholonomic system,. i.e. it only performs longitudinal motion without lateral or vertical motions. 
	\item[a.2]
	The IMU and the compass are co-located, pre-calibrated, and fixed at the vehicle geometric center. 
    \item[a.3]
     The IMU, compass, and velocity readings are synchronized and time-stamped. 
\end{enumerate}

As part of the input of the problem, a prior road map consisting of a set of roads with GPS waypoints is required. The typical distance between adjacent waypoints is around $20$m. Common notations are defined as follows,
\begin{itemize}
	\item $\mathcal{M}_p:=\{\mathbf{x}_{m}=[x_m,y_m]^\mathsf{T}\in\mathbb{R}^2 \vert m \in \mathscr{M}\}$ represents the prior road map which is a set of GPS positions where $\mathscr{M}$ is the position index set.  Note that these GPS positions are map points instead of live GPS inputs. We do NOT use GPS receiver in our algorithm design.
	%\item $\mathcal{R}=\{\mathcal{R}_{i_r}\vert i_r \in \mathcal{I}_r\}$ represent roads on $\mathcal{M}_p$ where $\mathcal{I}_r$ is the road index set. Each $\mathcal{R}_{i_r}:=\{\mathbf{x}_{m}=[x_m,y_m]^\mathsf{T} \vert m \in\mathscr{R}(i_r)\}$ is a set of ordered GPS positions where $\mathscr{R}(i_r)$ collects all indexes of positions on $\mathcal{R}_{i_r}$. $\mathcal{R}_{i_r}\subset \mathcal{M}_p$. 
	\item $\mathbf{a}=\{\mathbf{a}_{j}\in\mathbb{R}^3|j = 0, 1,\cdots, N_j\}$ and $\omega=\{\mathbf{\omega}_{j}\in\mathbb{R}^3|j = 0, 1,\cdots, N_j\}$ 
	denote accelerometer readings and gyroscope angular velocities from the IMU, respectively. 
    \item $\phi=\{\phi_{j_{\phi}}\in \mathbb{R}|j_{\phi} = 0, \cdots, \lfloor\frac{N_j}{c_{\phi}}\rfloor\}$ denotes compass readings where $c_{\phi}\geq 1$ since a compass often has lower sampling frequency than that of the IMU.
	\item $\mathbf{v}=\{v_{j_{v}}\in \mathbb{R}|j_{v} = 0, \cdots, \lfloor\frac{N_j}{c_v}\rfloor\}$ denotes wheel speed readings from OBD where $c_{v}\geq 1$ because it has a lower sampling frequency than that of IMU. And $v_{j_{v}}$ is the speed at midpoint of car rear wheels. 
\end{itemize}
The GBPL problem is defined as follows.
\begin{Pro}
Given $\mathcal{M}_p$, 
%$\mathcal{R}$, 
$\mathbf{a}$, $\omega$, $\phi$ and $\mathbf{v}$, localize the robot after its heading changes. As its localized, report robot location continuously.
\end{Pro}

%%%%%%%%%%%%%%%%%%%%%%%%%%%%%%%%%%%%%%%%%%%%%%%%%%%%%%%%%%%%%%%%%%%%%%%%%%%%%%%%
\section{GBPL Modeling and Design} \label{sec:plam_model}

Our system diagram is illustrated in Fig.~\ref{fig:SystemDiagram} which consists of four main building blocks: HLG construction, query sequence generation (QSG) thread, global localization (GL) thread, and location alignment and verification (LAV) thread. HLG construction is shaded in light gray which converts the prior geographic map into an HLG which runs only once in advance. For the rest shaded in dark gray, we refer to them as threads because they can be implemented as a parallel multi-threaded application. The QSG thread runs EKF constantly at the back end as the system receives sensory readings $\mathbf{a}$, $\omega$, $\phi$ and $\mathbf{v}$ and outputs the estimated trajectory. GL thread searches for the global location on a turn-by-turn basis. GL thread performs Bayesian graph matching between the query sequence extracted from the trajectory and the HLG. After the global location is obtained, GL terminates and LAV aligns the latest segment with the map and uses the result to rectify error drifting in the EKF in QSG. If no satisfying alignment is found, LAV terminates and the system restarts GL. In fact, GL thread and LAV thread work alternatively depending on whether the robot is localized or not. We begin with HLG construction.

\subsection{HLG Construction}\label{ssec:GraphConstruction}
We pre-process map $\mathcal{M}_p$ to construct an HLG to facilitate heading-length matching. There are three reasons for using HLG instead of matching on $\mathcal{M}_p$ directly. 
\begin{itemize}
\item First, the vehicle trajectory may not exactly match with $\mathcal{M}_p$. Since $\mathcal{M}_p$ and most maps do not have lane-level information, the discrepancy between the estimated trajectory and $\mathcal{M}_p$ is non-negligible which makes the direct trajectory-to-map matching unreliable. Fig.~\ref{fig:MapTrajDiff} shows an example. For the same route, the trajectories may be different due to driving on different lanes, driver habit, traffic, etc. \item Second, matching trajectory with $\mathcal{M}_p$ directly is computationally expensive because the searching space grows with the total number of GPS waypoint positions in $\mathcal{M}_p$. 
\item Third, the inevitably accumulated trajectory drift deteriorates the matching quality and makes the matching unreliable. 
\end{itemize}

\begin{figure}[th!]
 	\includegraphics[width=0.8\linewidth, viewport=19 261 984 764, clip=true]{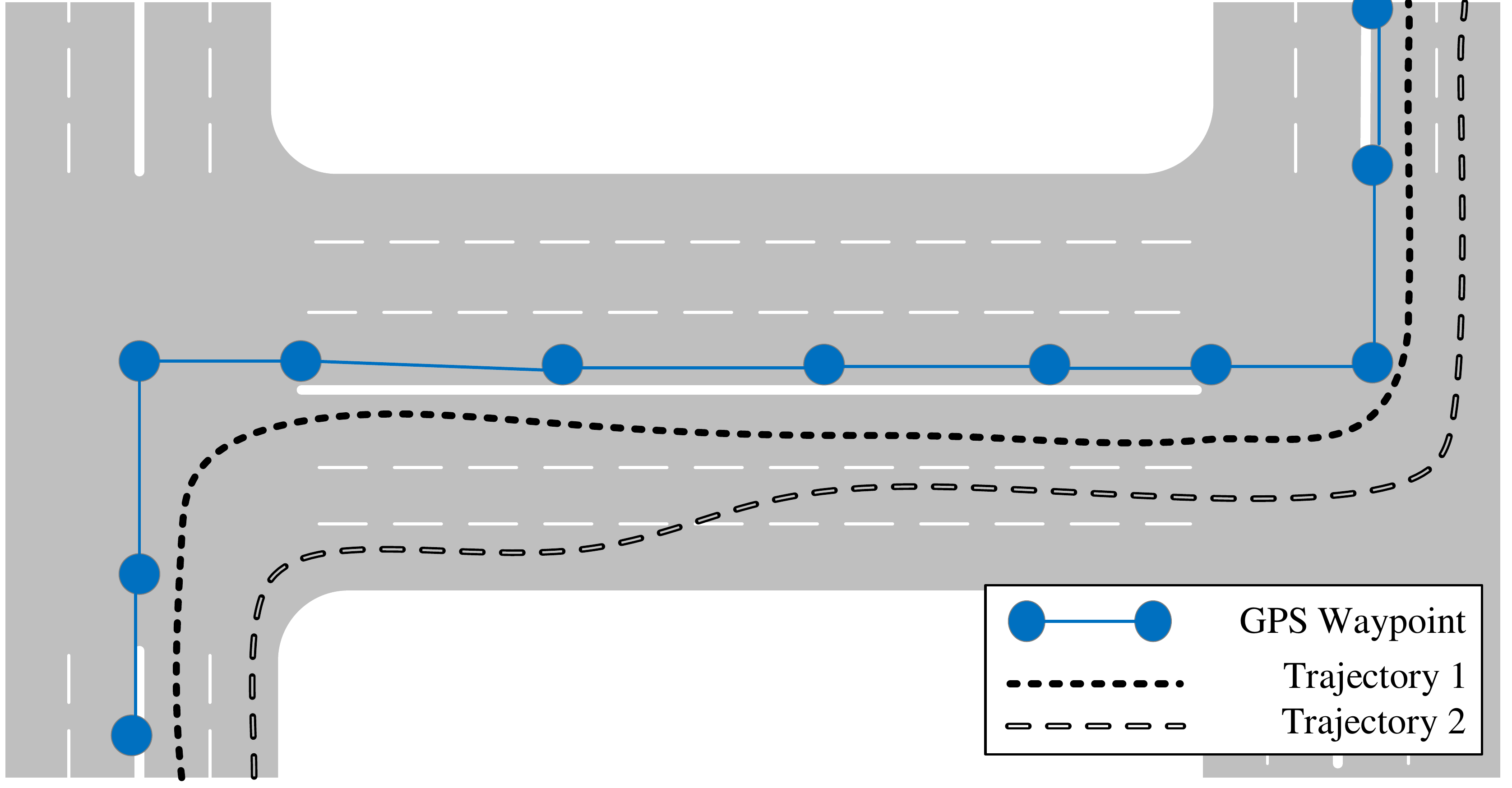}
 	\caption{Map and trajectory discrepancy illustration. Given the trajectory generated by proprioceptive sensors, directly matching trajectory with the map may not be desirable. For the same route, trajectories 1 and 2 appear quite differently. Neither of them matches blue waypoints in the map.}
 	\label{fig:MapTrajDiff}
\end{figure}

\begin{figure}[ht!]
 	\includegraphics[width=1\linewidth, viewport=142 840 752  1073, clip=true]{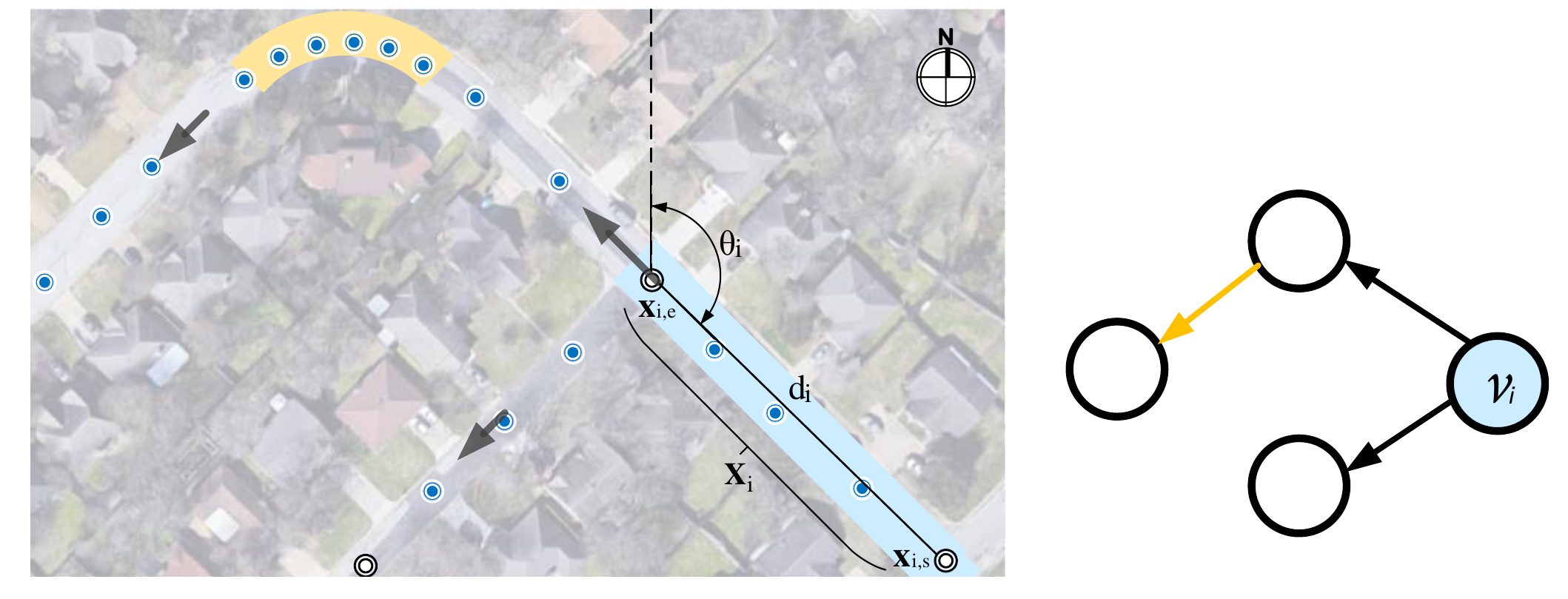}
 	\caption{HLG illustration in color. The left figure shows a satellite image with road map consisted of GPS waypoints (blue dots) overlaying on top of the image and intersections represented in small black circles. We estimate road curvature changes to capture heading change and construct HLG. As an example, we color a long and straight segment with light blue and a curve segment with light orange. The right figure shows the corresponding HLG, and we only employ long road segment vertices for localization. }
 	\label{fig:HLG}
\end{figure}
 
Therefore, we extract features from the map which are the long straight segments and represent them as the HLG. This leads to a graph matching approach that can mitigate the influence of the aforementioned three issue.  We start with HLG construction based on our prior work~\cite{cheng2018plam} where we have estimated road curvature changes to capture orientation change and construct a heading graph (HG). Build on~\cite{cheng2018plam}, we augment length information in HG to construct HLG for heading-length matching. Fig.~\ref{fig:HLG} illustrates an example. For completeness, we provide an overview here and more detail description of constructing the graph can be found in~\cite{cheng2018plam}. We denote the HLG by a directed graph $\mathcal{M}_{h}=\{\mathcal{V}_h,\mathcal{E}_h\}$ where $\mathcal{V}_h$ is the vertex set and $\mathcal{E}_h$ and is the edge set. A vertex $v_i\in\mathcal{V}_h$ represents a straight and continuous road segment with neither orientation changes nor intersections. An edge $e_{i,i'}\in \mathcal{E}_h$ captures the connectivity between nodes and characterizes the orientation change between the two connected vertices $v_i$ and $v_{i'}$. $\mathcal{M}_{h}$ has two types of edges: road intersections and curve segments; and two types of vertices: long straight segment vertices and short transitional segment vertices. The short transitional segment vertices are often formed between curve segments or curved roads entering intersections. 
 
 To build $\mathcal{M}_h$, we split each road at road intersections and further segment them into two types of segments to capture orientation changes: straight segments and curved segments~\cite{cheng2018plam}. With all roads segmented, we compute orientation and length for vertices corresponding to those long straight road segments. Each vertex contains the following information 
 \begin{equation}
 v_i=\{\mathbf{X}_i,\theta_i,d_i,b_i\},
 \end{equation} 
 where  $\mathbf{X}_i=[ \mathbf{x}_{i,s}^\mathsf{T},\cdots,\mathbf{x}_{i,e}^\mathsf{T}]^\mathsf{T}$ contained all 2D waypoint positions in GPS coordinates of the road segment with starting position $\mathbf{x}_{i,s}$ and ending position $\mathbf{x}_{i,e}$, orientation $\theta_i\in(-\pi,\pi]$ is the angle between the geographic north and the orientation of the road segment computed using $\mathbf{X}_i$ with a least squares estimation method adopted from~\cite{cheng2018plam}, $d_i$ is road segment length which is computed. by
 \begin{equation}
 d_i = ||\mathbf{x}_{i,s}-\mathbf{x}_{i,e}||,
 \label{eq:di}
 \end{equation}
 and $b_i$ is the binary variable indicate if the vertex is a long road segment. We only perform orientation estimation if $d_i > t_l$ where $t_l$ is the threshold for road segment length. That is,
 \begin{equation}
 b_i = 
 \begin{cases}
 1,~~d_i > t_l,\\
 0,~~\mbox{otherwise}.\\
 \end{cases}
 \end{equation}
 Only long road segments ($b_i=1$) will be used in localization which defines vertex subset $\mathcal{V}_{h,l} \subseteq \mathcal{V}_h$ corresponding to long straight segments. Note that $\theta_i$  depends on the robot traveling direction and hence $\mathcal{M}_{h}$ is a directed graph. 
 
 The errors of GPS waypoints in each entry of $X_i$ affect the accuracy of $\theta_i$ and $d_i$. To track map uncertainties caused by GPS errors, we derive the distribution of $\theta_i$ and $d_i$ using error variance propagation analysis~\cite{hartley2003multiple}. We model GPS errors by using Gaussian distribution and assuming GPS measurement noises to be independent and identically distributed. We denote the GPS measurement variance by $\sigma^2_g$. According to~\cite{Brubaker16selfloc}, typical consumer grade navigation systems offer positional accuracy around $\sigma_g=10$m. The distribution of $\theta_i$ that characterizes its uncertainty is 
 \begin{equation}
 \theta_i\sim\mathcal{N}(\mu_{\theta_i},\sigma^2_{\theta_i}),
 \label{eq:dist_thetai} 
\end{equation} 
 where $\sigma^2_{\theta_i}$ is derived in~\cite{cheng2018plam}. And the distribution of $d_i$ is  
 \begin{equation}
 d_i\sim\mathcal{N}(\mu_{d_i},\sigma^2_{d_i})=\mathcal{N}(\mu_{d_i},2\sigma^2_{g}).
 \label{eq:dist_di} 
 \end{equation}

 \subsection{Query Sequence Generation (QSG) Thread}\label{sec:EKFtrajory}
 To localize the vehicle on $\mathcal{M}_h$, we estimate the trajectory from sensory readings with an EKF-based approach. We then generate a discrete query consisting of a heading-length sequence extracted from the EKF trajectory results. It is worth noting that our method is not sensitive to the global drift of the EKF estimated trajectory because we only use short segmented trajectory to extract heading and length of its straight segments.
 % To improve... (robustnes)  To improve the robustness, we only keep headings when the vehicle is traveling on long and straight road segments. This means the headings should be stable and constant in a long stretch of travel time. %To obtain the query sequence: (1) we first utilize an EKF-based approach using $\omega_{0:t}$ and $\phi_{0:t_{\phi}}$ to estimate heading, 
 %(2) sequence segmentation to obtain stable orientations, and (3) remove headings that do not correspond to long and straight road segments.
 %In ET thread, we have sensory data from IMU, compass and wheel velocity to estimate vehicle trajectory using EKF and also do turn detections. Additionally, with EKF framework, 
 %we track the uncertainty of estimated angles and travel distances. 

 \subsubsection{EKF-based Trajectory Estimation}\label{ssec:ETE}
 Note that readings from the IMU, the digital compass, and the vehicle velocity: $\mathbf{a}$, $\boldsymbol{\omega}$, $\boldsymbol{\phi}$, and $\mathbf{v}$, are the inputs to the EKF-based approach to estimate vehicle trajectory~\cite{bar2004estimation,Yi2009SkidEKF,cheng2017localization}. To start the EKF, we need a stabilized initial compass reading $\phi_0$ to determine the initial vehicle orientation which can be obtained by driving on a long and straight segment of road (Assumption a.0). We define two right-handed coordinate systems: IMU/compass device body frame $\{B\}$ (also overlapping with vehicle geometric center), the fixed inertial frame $\{I\}$ which shares its origin with $\{B\}$ at the initial pose. Frame $\{I\}$'s $X$-$Y$ plane is a horizontal plane parallel to the ground plane with $Y$ axis pointing to magnetic north direction and $Z$ axis is vertical and points upward. In the state representation, let state vector $\textbf{X}_{s,j}$ at time $j$ be:
 \begin{equation}\label{eq:EKFstate}
 \textbf{X}_{s,j}:=[\mathbf{p}^{I}_j,\mathbf{v}^{I}_j,\mathbf{\boldsymbol{\Theta}}^{I}_j, s_j]^{\mathsf{T}},
 \end{equation}
 which includes position $\mathbf{p}^{I} =[x, y, z]^\mathsf{T}\in \mathbb{R}^3$, velocity $\mathbf{v}^{I} =[\dot{x},\dot{y},\dot{z}]^\mathsf{T}\in\mathbb{R}^3$, and the Euler angles $\boldsymbol{\Theta}^{I}:=[\alpha,\beta ,\gamma]^{\mathsf{T}}$ in $\{I\}$ in $X\mbox{-}Y\mbox{-}Z$ order, and scale/slip factor (SSF) $s$. We define $s$ here to address vehicle velocity error which can be caused by tire radius error such as inflation level, road slippery, etc. The superscripts indicate in which frame the vector is defined. The transformation from $\{I\}$ to $\{B\}$ is the $Z\mbox{-}Y\mbox{-}X$ ordered Euler angle rotation. 
 The state transition equations are described as follows:
 \begin{equation}\label{eq:kine_eq}
 \begin{aligned}
 \mathbf{p}^{I}_j &= \mathbf{p}^{I}_{j-1} +\tau_{\omega}\mathbf{v}^{I}\\
 \mathbf{v}^{I}_j &= \mathbf{v}^{I}_{j-1} +\tau_{\omega}(\prescript{I}{B}{\mathbf{R}}(\mathbf{a}) - \textbf{G}) \\
 \Theta_{j} &= \Theta_{j-1}+\tau_{\omega}\prescript{I}{B}{\mathbf{E}}(\boldsymbol{\omega})+\mathbf{c}_{\gamma} \\
 s_{j} & = s_{j-1},
 \end{aligned}
 \end{equation}
where $\tau_{\omega}$ is the IMU sampling interval, $\textbf{G}=[0 \ 0 \ -9.8]^\mathsf{T}$ is the gravitational vector, $\mathbf{c}_{\gamma}=[0 \ 0 \ \phi_0]^\mathsf{T}$ is the initial orientation determined by $\phi_0$, $\prescript{I}{B}{\mathbf{R}}$ is the rotation matrix from $\{B\}$ to $\{I\}$, and $\prescript{I}{B}{\mathbf{E}}$ is the rotation rate matrix from $\{B\}$ to $\{I\}$.

For EKF observation models, we use velocity constraint from vehicle movement, sensory readings $\boldsymbol{\phi}$ and $\mathbf{v}$, and estimated scale by matching trajectory with map which will be discussed in Section~\ref{ssec:ScaleAdjustment}. First, according to Assumptions a.1 there is no lateral or vertical movements in $\{B\}$, the velocities along $Y$ axis and $Z$ axis in $\{B\}$ are set to be zeros. The velocity constraint is written as:
\begin{equation} \label{eq:obs_velconstraint}
(\prescript{B}{I}{\mathbf{R}})_{2:3} \mathbf{v}^{I}_j =
\begin{bmatrix}
0 &0
\end{bmatrix}
^\mathsf{T},
\end{equation}
where $\prescript{B}{I}{\mathbf{R}}_{2:3}$ is the second and third rows of $\prescript{B}{I}{\mathbf{R}}$. 

From the coordinate definition, the heading direction is $\gamma$ defined in $\{I\}$ (last component of $\boldsymbol{\Theta}^{I}$), we take compass reading $\boldsymbol{\phi}$ as its observation. In our physical system, compass readings have a lower sampling frequency than that of the IMU readings, we use the latest available reading. Also, compass readings may be polluted by other magnetic fields, we can recognize faulty readings by cross-validating compass readings with IMU readings. We discard the faulty compass readings if the difference between the estimated heading state and the compass reading exceeds an threshold. With the cross-validated compass reading, we update heading direction $\gamma$ by 
\begin{equation} \label{eq:obs_phi}
\gamma_j = 
\begin{cases}
\phi_{j_\phi},  &\mbox{~if~$j=c_\phi j_{\phi}$}\\
\phi_{j_\phi-1}, &\mbox{~otherwise.}
\end{cases}
\end{equation}
	
We compensate SSF $s_j$ by estimating its value from aligned map data after taking a turn. We will detail 
how to compute $s_{ssf}$ and its variance in Section~\ref{ssec:ScaleAdjustment}. For $s$, we have 
\begin{equation} \label{eq:obs_scale}
s_j = s_{ssf},
\end{equation}
where $s_{ssf}$ is  the  ratio of the trajectory length from the map versus that from the query. Lastly, we take wheel velocity $\mathbf{v}$ as observations. Similar to $\phi$ that the sampling frequency is lower than IMU readings, we have 
 \begin{equation} \label{eq:obs_obdvel}
 ||\mathbf{v}^{I}_j||=
 \begin{cases}
 s_jv_{j_v}, &\mbox{~if~$j=c_v j_v$} \\
s_jv_{j_v-1}, &\mbox{~otherwise.}
 \end{cases}
\end{equation}
 Combining \eqref{eq:obs_phi}, \eqref{eq:obs_velconstraint}, \eqref{eq:obs_obdvel}, and \eqref{eq:obs_scale}, we complete the observation model functions. The rest is to follow the standard EKF setup. Fig.~\ref{fig:EKF-gps-traj} shows the estimated EKF trajectory  compared with the corresponding GPS ground truth trajectory. Note that the vehicle takes some additional turns to assist localization (Assumption a.0) and the trajectory is not the shortest. 
 
 % state/ covariance initialization
 
 \begin{figure}[ht!]
 	\subfigure[]{\includegraphics[width=0.285\linewidth, viewport=161 243 432 540, clip=true]{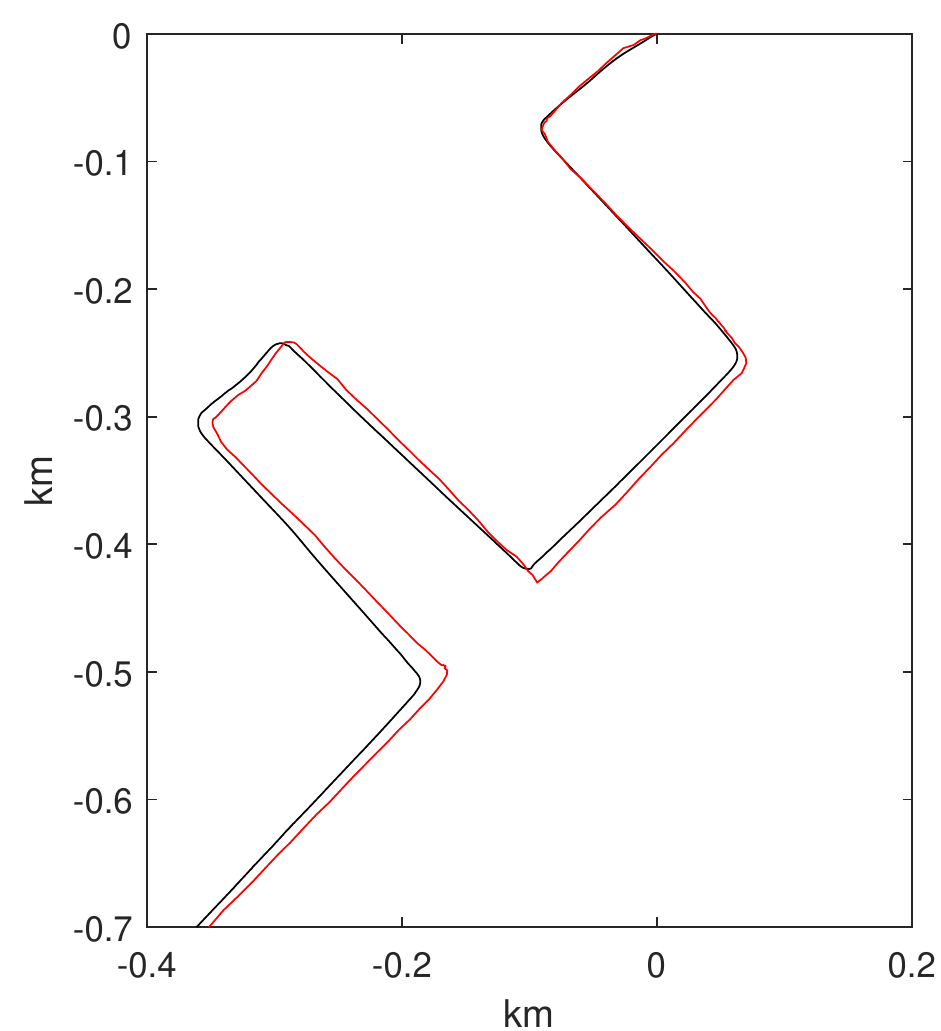}\label{fig:EKF-gps-traj}}
 	\subfigure[]{\includegraphics[width=0.40\linewidth, viewport=110 243 485 540, clip=true]{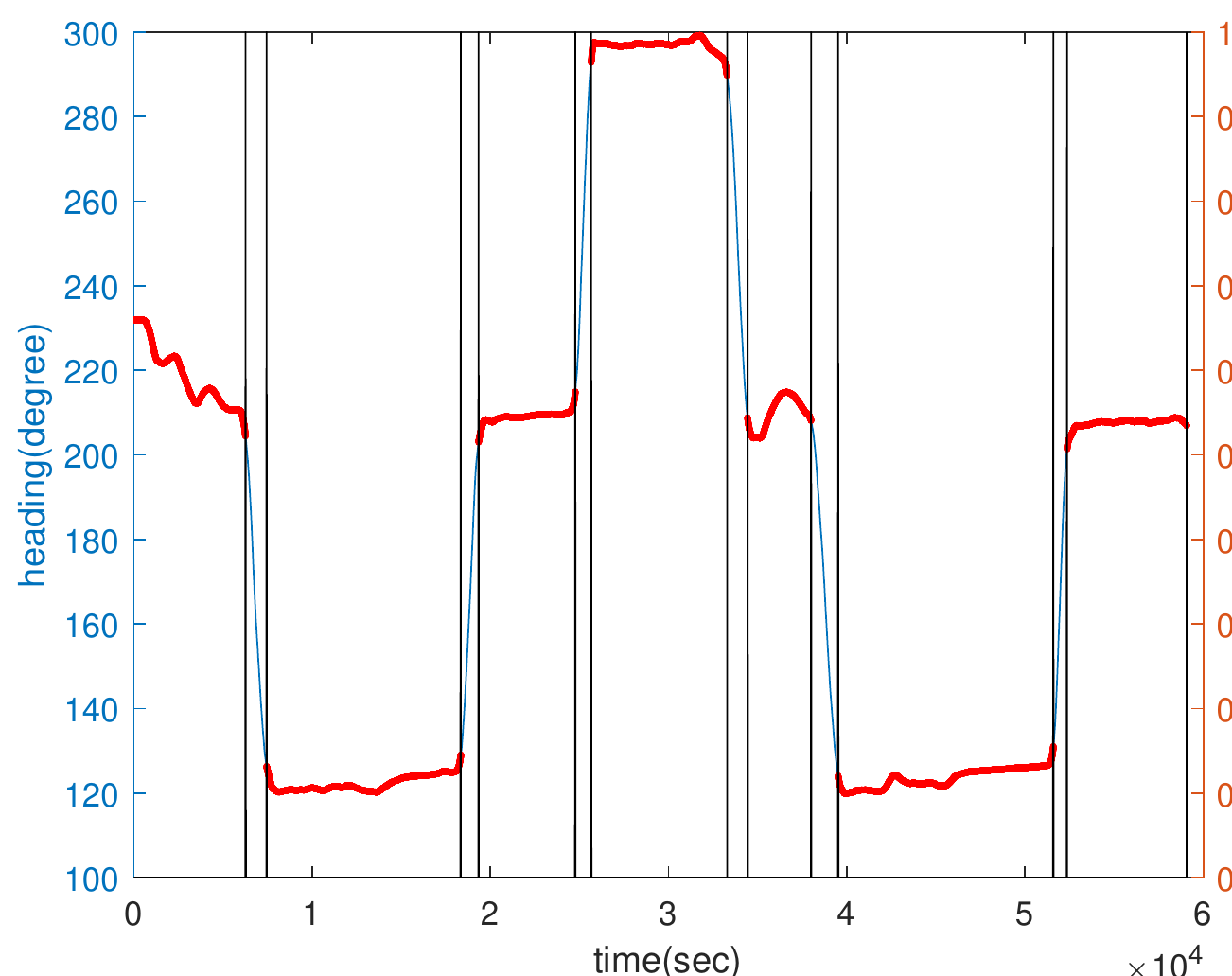}\label{fig:EKF-angle}}
 	\subfigure[]{\includegraphics[width=0.285\linewidth, viewport=161 243 432 540, clip=true]{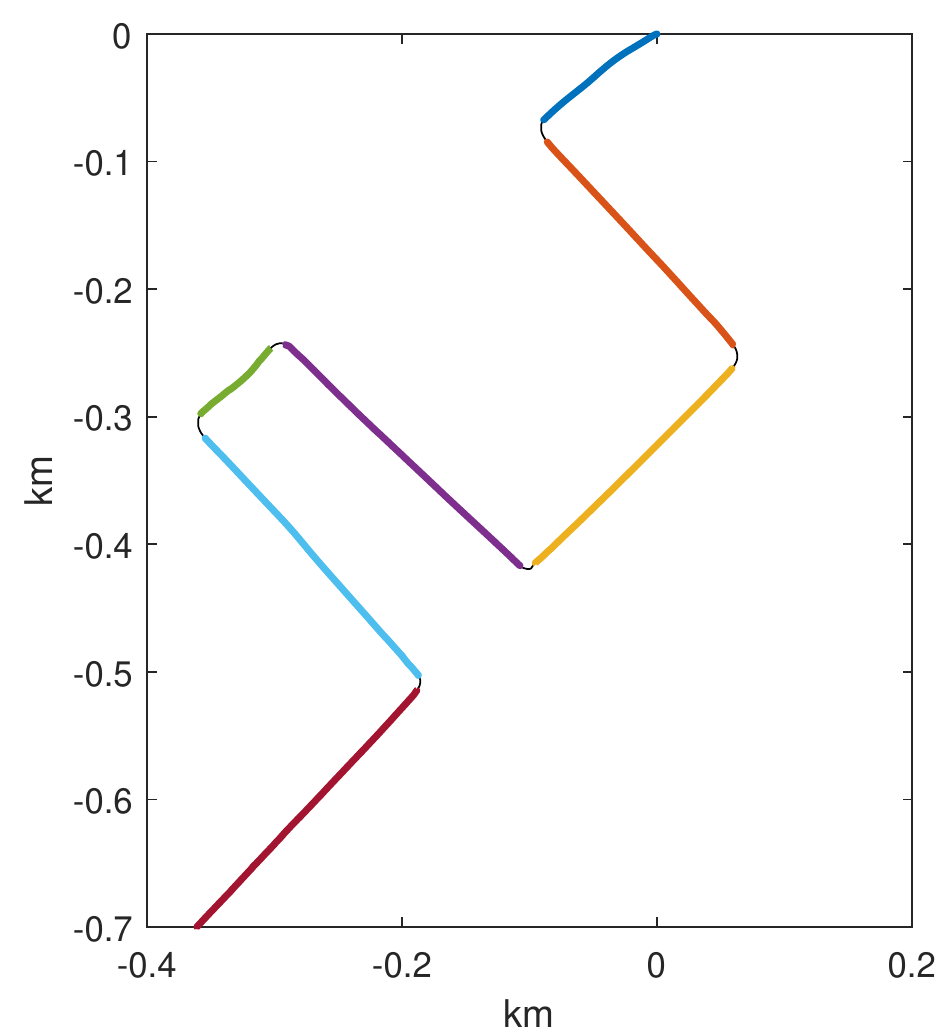}\label{fig:EKF-markedtraj}}
 	\caption{(a) Trajectory estimation result: the red line is the GPS ground truth, and black line illustrates the EKF estimated trajectory. (b) Query heading representations. Blue line is estimated heading, black vertical lines are indices where data segmented, red lines mark out stable heading segments and unmarked segments are detected turns. (c) Corresponding travel heading and length segment representations. Different segments are marked in different colors.}
 	\label{fig:EKF}
 \end{figure}

\subsubsection{Heading-Length Sequence Generation}\label{ssec:TurnDetection_QuerySeq}
With the estimated trajectory, we generate query heading-length sequence by capturing vehicle heading changes. We adopt the method for heading sequence generation from~\cite{cheng2018plam}
and augment corresponding length sequence in this work. To improve the robustness, we only keep headings when the vehicle is traveling on long and straight road segments. This means the headings should be stable and constant in a long stretch of travel time and corresponding travel distance is long. From the coordinate definition, the headings is $\gamma$ in $\{I\}$ and is denoted by $\gamma_{0:j}$. To obtain the query sequence, we segment $\gamma_{0:j}$ to get stable headings and remove false positive headings that do not correspond to long and straight road segments. 
In Fig.~\ref{fig:EKF-angle}, red horizontal segments are detected stable headings. Hence we obtain the set of query 
query heading sequence which is denoted by $\Theta_q = \{\Theta_{q,k}\vert k = 1,\cdots,n \}$ where $k$ is query data index, $n$ is the number of straight segments. Each subset $\Theta_{q,k}$ corresponding to continuous observations from EKF represents a straight segment. At the same time, we generate the corresponding travel length sequence which is denoted by $D_q=\{d_{q,k}\vert k = 1,\cdots,n\}$ where $d_{q,k}$ is the travel length of the segmented route (e.g. colored segments in Fig.~\ref{fig:EKF-markedtraj}). 

We denote the query heading-length sequence by $Q:=\{\Theta_{q},D_q\}$ which consists of the segmented heading-length sequence. The uncertainty of query sequence $Q$ is obtained from EKF variance estimation. For $\Theta_q$, we define $\overline{\theta_{q,k}}$ as the sample mean orientation of segment $\Theta_{q,k}$ which contains $n_{\theta_{q,k}}$ observations of random variable  $\theta_{q,k}$.  $\theta_{q,k}$ has it covariance matrix obtained from EKF. For $D_{q}$, the variance of can also be derived from EKF and we denote variance of $d_{q,k}$ by $\sigma^2_{d_{q,k}}$. Those variables will be used later in the analysis part.

It is worth noting that each entry of the sequence is not sensitive to the overall trajectory drift due to the localized computation. The resulting sequence also can be understood as local features for the trajectory. Also, reducing the query to the discrete feature sequence helps in reducing computation complexity.

\begin{comment}
To get stable headings from $\gamma_{0:j}$, we detect changing points of $\gamma_{0:j}$ by employing the sliding window algorithm \cite{keogh2001online} and segment it into a set of non-overlapping, consecutive segments\cite{cheng2018plam}. This filtering can help smooth out the situation when vehicles change lanes on a long and straight roads. A stable query heading sequence is generated by extracting segments with absolute slope smaller than $10^{-4}$ in our setting. 
%Also, by removing false positive segments  

Next, we remove headings that do not correspond to long and straight road segments. With the assistance of estimated trajectory, we can easily detect stable heading false positive cases such as vehicle stops or slowing driving over a short distance by checking the corresponding travel length. 
\end{comment}

\subsection{Global Localization Thread}\label{ssec:GlobalLoc}

\subsubsection{GL Overview}

With the query sequence obtained from on-board sensors, we are ready to match it with sequences on the HLG to search for the actual location. This is a graph matching problem. In the GL thread, we localize the robot when the robot changes its heading which is the moment the query sequence grows its length. It is worth noting that GL is an intermittent localization. The continuous localization will be address later in the paper.

Given the query heading-length sequence, we search for the best match of heading-length sequence in the HLG $\mathcal{M}_h$. For any long straight candidate vertex in $\mathcal{V}_{h,l}$, we match the query heading-length sequence with sequences of the vertices starting at the candidate vertex. We discard candidate vertices with poor matching. 
In each candidate sequence to query sequence matching, We model sensory and map uncertainties and formulate the matching process as a sequential hypothesis test problem. The result of GL depends on if a satisfying matching sequence can be found. 

%As the vehicle is localized, we terminate the GL thread and start the LAV thread. In cases the robot is lost, we restart the GL thread.

%% algorithm

\subsubsection{Graph Matching}\label{ssec:GraphMatching}
The center part of GL is the matching of query sequence and candidate sequence on the graph. To achieve this, we expand the heading sequence matching in \cite{cheng2018plam} to find the best heading-length matching in $\mathcal{M}_h$. Given query sequence $Q=\{\Theta_{q},D_{q}\}=\{(\theta_{q,k},d_{q,k})|k=1,\cdots,n\}$, let us denote a candidate heading-length vertex sequence in $\mathcal{M}_{h}$ by $M:=\{\Theta,D\}=\{(\theta_{k},d_{k})\vert  k = 1,\cdots,n\}$ correspondingly. 
As a convention in this paper, for random vector $\star$, $\mu_{\star}$ represents its mean vector. Following the convention, mean matrix of $Q$ is defined as $\mu_Q=[{\mu}_{\Theta_{q}}^\mathsf{T},{\mu}_{D_{q}}^\mathsf{T}]^\mathsf{T}$ where ${\mu}_{\Theta_{q}}=[\mu_{\theta_{q,1}},\cdots,\mu_{\theta_{q,n}}]^\mathsf{T}$ and 
${\mu}_{D_{q}}=[\mu_{d_{q,1}},\cdots,\mu_{d_{q,n}}]^\mathsf{T}$. The mean matrix of $M$ is denoted by $\mu_M=[{\mu}_{\Theta}^\mathsf{T},{\mu}_{D}^\mathsf{T}]^\mathsf{T}$ where $\mu_{\Theta}=[\mu_{\theta_{1}},\cdots,\mu_{\theta_{n}}]^\mathsf{T}$ and $\mu_{D}=[\mu_{d_{1}},\cdots,\mu_{d_{n}}]^\mathsf{T}$.
% The variance of $d_k$ is the sum of traveling vertex length variances.

\begin{figure}[b!]
	\subfigure[]{\includegraphics[width=0.49\linewidth, viewport=110 243 485 540, clip=true]{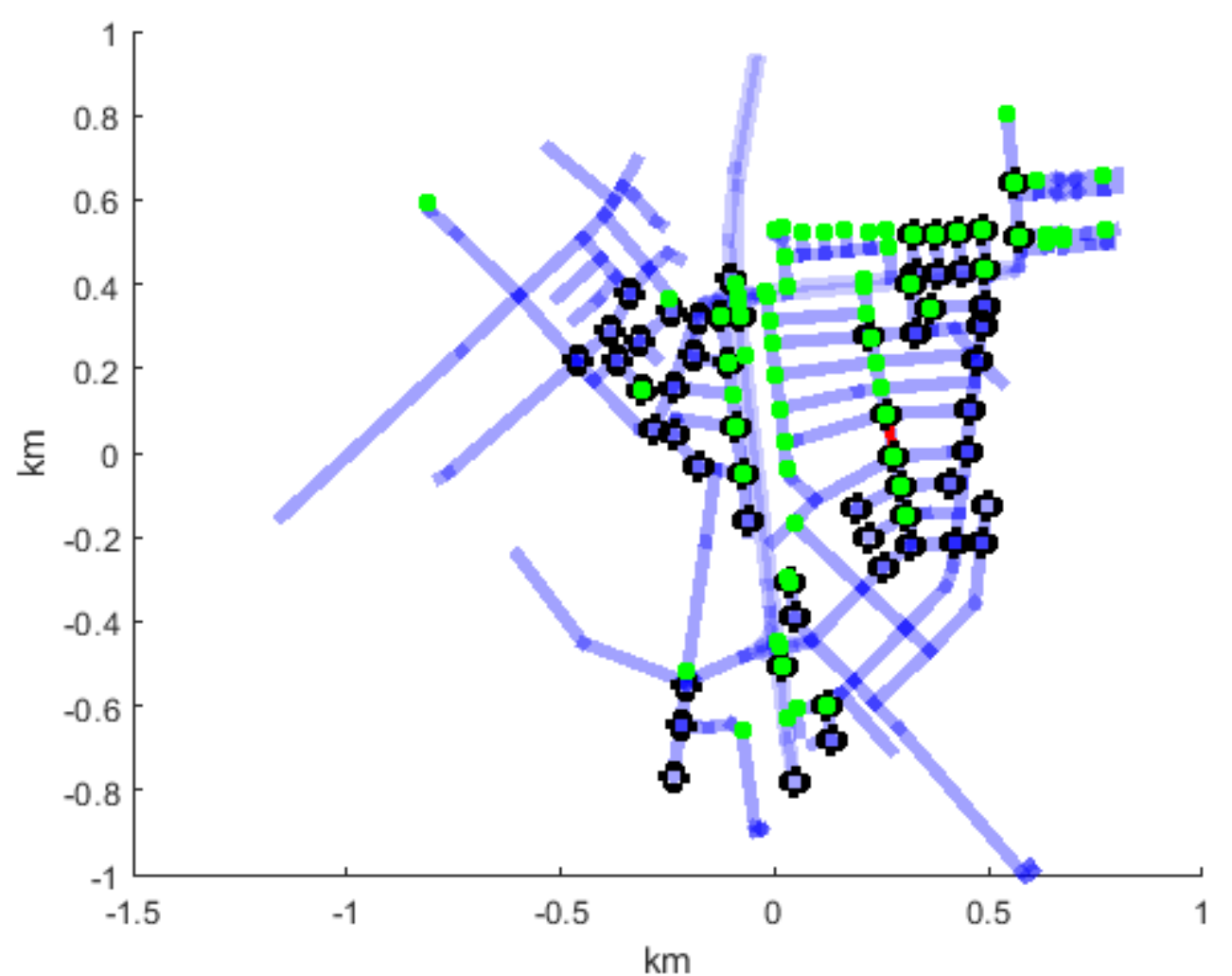}\label{fig:gloc_c1}}
	\subfigure[]{\includegraphics[width=0.49\linewidth, viewport=110 243 485 540, clip=true]{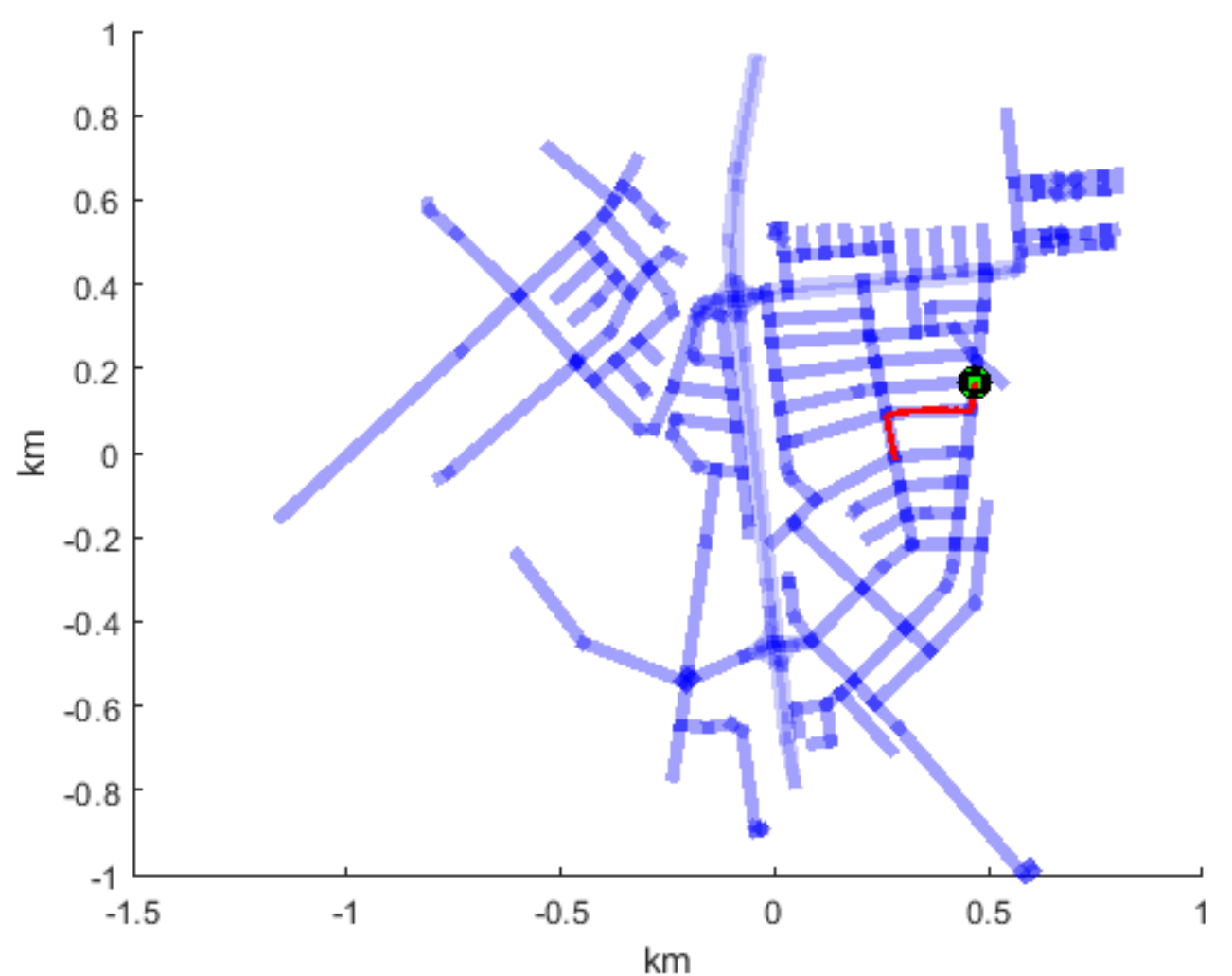}\label{fig:gloc_c3}}
	\caption{ An example of global localization. (a) The candidate locations using heading matching (green dots), length matching (black circle). We show that performing heading-length matching (locations with green dot and black circle) helps reducing candidates. (b) The candidate localization is reduced to the single solution if the joint distribution between heading and length is used.}
	\label{fig:GlobalLocalization}
\end{figure}

Due to independent measurement noises, the conditional matching probability between query sequence $Q:=\{\Theta_{q},D_{q}\}$ and a candidate sequence $M:=\{\Theta,D\}$ on HLG $\mathcal{M}_{h}$ is 
\begin{align} \label{eq:match_prob}
\notag&P(\mu_Q =\mu_M \vert Q,M) \\
&= P(\mu_{\Theta_q} =\mu_{\Theta}\vert \Theta_q, \Theta)P(\mu_{D_q} =\mu_{D}\vert  D_q, D).
\end{align}
From~\cite{cheng2018plam}, the conditional heading matching probability between $\Theta_q$ and $\Theta_h$ is
\begin{align} 
\label{eq:seq_heading_prob} P({\mu}_{\Theta_{q}} ={\mu}_{\Theta}\vert \Theta_{q},\Theta) \propto  \prod_{k=1}^{n}f_{T}(t(\theta_{q,k},\theta_k)),
\end{align}
due to independent sensor noises and $f_T(t(\theta_{q,k},\theta_k))$ is the probability density function (PDF) of Student's t-distribution. 
For length matching, the conditional matching probability between $D_q$ and $D$ is 
\begin{equation}\label{eq:normal_dist}
P(\mu_{D_q}=\mu_{D}\vert D_q,D)\propto \prod_{k=1}^{n} f(z(d_{q,k}, d_k)), 
\end{equation}
where $f(\cdot)$ is the PDF of standard normal distribution, and 
$z(d_{q,k},d_k)=\frac{d_{q,k}-d_k}{\sqrt{\sigma^2_{d_{k}}+\sigma^2_{d_{q,k}}}}.$
Combining~\eqref{eq:seq_heading_prob} and \eqref{eq:normal_dist} and recalling that $n$ is the number of straight segments in the query sequence, we rewrite \eqref{eq:match_prob} as follows,
\begin{align}\label{eq:seqm_prob}
P(\mu_Q =\mu_M \vert Q,M)\propto \prod_{k=1}^{n}f_{T}(t(\theta_{q,k},\theta_k))f(z(d_{q,k}-d_k)).
\end{align}

\subsubsection{Candidate Vertex Selection}\label{ssec:VertexThreshold}
To select on candidate vertices during matching, we perform statistical hypothesis testing to remove unlikely matchings. According to~\eqref{eq:match_prob}, sequence matching is considered as multiple pair matching. For each pair $(\{\theta_k,d_k\},\{\theta_{q,k},d_{q,k}\})$, it is a hypothesis testing 
\begin{eqnarray} 
&\mathbf{H_{0}}:& [\mu_{\theta_{q,k}}, \mu_{d_{q,k}}]^\mathsf{T} =[\mu_{\theta_{k}},\mu_{d_{k}}]^\mathsf{T}\nonumber\\
&\mathbf{H_{1}}:& \mbox{~otherwise.}  \label{eq:hypo_vertexmatch}
\end{eqnarray}
Hypothesis $H_0$ can be seen as two null hypotheses: $H_{0,\theta}:\mu_{\theta_{q,k}}=\mu_{\theta_{k}}$ and $H_{0,d}:\mu_{d_{q,k}}=\mu_{d_{k}}$. We perform two individual tests separately with significance level $1-\alpha$ where $\alpha$ is a small probability. Both $H_{0,\theta}$ and $H_{0,d}$ are two-tailed distributions. We choose $t_{\alpha/2,\nu}$ as the t-statistic with a cumulative probability of $(1-\frac{\alpha}{2})$ where $\nu$ is the degrees of freedom (DoF) and $z_{\alpha/2}$ as the z-statistic with a cumulative probability of $(1-\frac{\alpha}{2})$. We reject $H_0$ if 
\begin{equation}\label{eq:t_select}
(|t(\theta_k,\theta_{q,k})|>t_{\alpha/2,\nu}) \vee (|z(d_k,d_{q,k})|>z_{\alpha/2}).
\end{equation} 
By sequentially applying the hypothesis testing on each corresponding pair $(\{\theta_k,d_k\},\{\theta_{q,k},d_{q,k}\})$ from query sequence $Q$ and candidate sequence $M$ on HLG $\mathcal{M}_{h}$, we determine whether $M$ represents the actual trajectory.  Fig.~\ref{fig:GlobalLocalization} has shown that using the joint distribution of heading and length significantly reduce the number of solutions in the matching process. 

In the matching process, we might get many candidate solutions because the hypothesis test is conservative in rejection. To address the problem and check if we converge to a unique solution, we classify the computed probabilities of (\ref{eq:match_prob}) into two groups using the Ostu method\cite{otsu1979threshold}. The number of solutions is the group size. If the group with higher probability has only one candidate then the vehicle is localized. Otherwise, it means that the group with higher probability contains several trajectories with higher probabilities. It indicates that more observations are needed to localize the vehicle.

\subsubsection{GL Algorithm}

We summarize the heading-length matching method in Algorithm~\ref{alg:algo1}. In a nutshell, as we sequentially match the vertex down the query sequence, we compare it with the out-neighbor of remaining vertices on the graph using breadth-first search. 

Note that vertex $v_i$ may have adjacent vertices with same orientation. For example, consider the vehicle reaches a long straight road (with road intersections).
This long straight road corresponds a set of vertices with same orientation. We denote the set of straight path start from $v_i$ by $\mathcal{V}_s$. 

To reuse the computed information as the query sequence grows, we define the candidate vertex information set $\mathcal{C}_k$ where $k = 1,~\cdots, n$ is the length of the query sequence. The candidate vertex set is denoted by
$\mathcal{C}_k=\{\{v_i,\mathcal{V}_{M,i},p_i\}|i= 1,\cdots, n_{\mathcal{C}_k}\},$ where each element in $\mathcal{C}_k$ record the candidate vertex $v_{\mathtt{i}}$ (the starting vertex of the trajectory/path), $\mathcal{V}_{M,i}$ is the set of vertex path, and the matching probability $p_i$ in \eqref{eq:match_prob} and $n_{\mathcal{C}_k}$ is the cardinality of $\mathcal{C}_k$.
\begin{comment}
The candidate vertex information set is denoted by $\mathcal{C}_k :=\{(v_{i'},p_{i'},\mathcal{P}_{i'})|i'= 1 \cdots n_{\mathcal{C}}\}$, where $v_{i'}$ is the candidate vertex (the starting vertex of the trajectory/path), $\mathcal{P}_{i'}$ is the set of vertex path, and $p_{i'}$ is the matching probability for this path.
\end{comment}
To initialize, we set $\mathcal{C}_0:=\{\{v_i,\emptyset,\frac{1}{|{V}_{h,l}|}\}|i=1,\cdots ,|\mathcal{V}_{h,l}|\}$ because each vertex in $\mathcal{V}_{h,l}$ is equally likely to be the path starting vertex.
The computational complexity of calculating each term in \eqref{eq:match_prob} is $O(1)$ using the alias sampling method~\cite{kronmal1979alias}. The upper bound of candidate vertex cardinality is $\vert\mathcal{V}_{h,l}\vert$ and thus it takes $O(|\mathcal{V}_{h,l}|)$ to compute probability of all candidate vertices. The size of straight path set takes $O(|\mathcal{V}_{s}|)$ which is related to variation of map road headings in Sec.~\ref{ssec:mapentropy}. With little variation in headings (e.g. Manhattan streets), $|\mathcal{V}_{s}|$ is larger. On the contrary, $|\mathcal{V}_{s}|$ is small compared to $|\mathcal{V}_{h,l}|$ with large variation in road headings. In this case, $O(|\mathcal{V}_{s}| = O(1)$. The classification of probabilities into two groups is $O(|\mathcal{V}_{h,l}|)$ using Hoare's selection algorithm. 

We summarize the computational complexity of Algorithm~\ref{alg:algo1} in Lemma~\ref{lem:complexity}. 

\begin{Lem}
	The computation complexity of the heading-length matching is $O(n|\mathcal{V}_{s}||\mathcal{V}_{h,l}|)$.
	\label{lem:complexity}	
\end{Lem}

\begin{algorithm} [t!]
	\caption{Heading-length Graph Matching} \label{alg:algo1}
	\KwIn{ $\mathcal{M}_{h}=\{\mathcal{V}_h,\mathcal{E}_h\}$,~$\mathcal{C}_{k-1}$ and  $\{\theta_{q,k},d_{q,k}\}$ }
	\KwOut{$\mathcal{C}_{k}$ or vehicle location}
	\algsetup{linenosize=\small}
	\scriptsize
	\SetKwFunction{FMain}{HLM}
	\DontPrintSemicolon    
	$\mathcal{C}_0:=\{\{v_i,\emptyset,\frac{1}{|{V}_{h,l}|}\}|i=1,\cdots ,|\mathcal{V}_{h,l}|\}$                                            \Comment*[r]{$O(1)$} 
	\For (\Comment*[f]{$O(n)$}){$k=1,\cdots,n$}{
		$\mathcal{C}_k \leftarrow\emptyset$;   \Comment*[r]{$O(1)$}
		\For (\Comment*[f]{$O(|\mathcal{V}_{h,l}|)$}){$i = 1,\cdots,n_{C_{k-1}}$}{
			\uIf{$k==1$}{
				Access straight path set $\mathcal{V}_{s}$ start from $v_i$;                        \Comment*[r]{$O(1)$}
			}\Else{
				$v_{i'}\leftarrow$ last vertex in path $\mathcal{V}_{M,i}$ \Comment*[r]{$O(1)$}   
				$V_{i'}\leftarrow $ adjacent  verteices of $v_{i'}$ ( with different angles); \Comment*[r]{$O(1)$}  
				Access straight path set $\mathcal{V}_{s}$ start from each vertex in $V_{i'}$;  \Comment*[r]{$O(1)$}
	    	}
			\For(\Comment*[f]{$O(|\mathcal{V}_{s}|)$}){ $V_{s} \in \mathcal{V}_{s}$}{ 
				Access $\theta_{s}$ and $d_{s}$ of $V_{s}$;   \Comment*[r]{$O(1)$} 
				compute $p \leftarrow f_{T}(t(\theta_{s}, \theta_{q,k}))f(z(d_{s}, d_{q,k}))$                              \Comment*[r]{$O(1)$} 
				\uIf{Pass hypothesis testing in \eqref{eq:hypo_vertexmatch}}{
					Update matching probability $p_{i'}\leftarrow p_i\cdot p$                                   \Comment*[r]{$O(1)$}  
					$V_{M,i'} \leftarrow$ Append $\mathcal{V}_{s}$ to $\mathcal{V}_{M,i}$           \Comment*[r]{$O(1)$} 
					$\mathcal{C}_k\leftarrow \mathcal{C}_k \cup \{v_{i}, \mathcal{V}_{M,i'}, p_{i'}\}$  \Comment*[r]{$O(1)$}                             
				}			
			}	                                            	
		} 
		Classify probabilies in \eqref{eq:match_prob} of $C_k$ using Otsu's method;                         \Comment*[r]{$O(|\mathcal{V}_{s}||\mathcal{V}_{h,l}|)$}
		Remove group in $C_k$ with lower probabilities;                                                           \Comment*[r]{$O(1)$} 
		\uIf{$|C_k|>1$}{
		Return $C_k$;            \Comment*[r]{$O(1)$}
		}\Else{
		Set $I_G = 1$;           \Comment*[r]{$O(1)$}
		Return vehicle location; \Comment*[r]{$O(1)$}
		}   		
	}
\end{algorithm}

\subsubsection{Localization Analysis} The remaining problem is whether this sequence of hypothesis testing would converge to the true trajectory as the length of the sequence grows. To analyze this, let us define three binary events:  $A_k=1$ if $\mu_{d_{q,k}}=\mu_{d_{k}}$, $B_k=1$ if $\mu_{\theta_{q,k}}=\mu_{\theta_{k}}$, and $C_k=1$ if vertex $k$ in $M_{h}$ is the actual location. The joint event $C_1\cdots C_n=1$ is to say $M:=\{\Theta, D\}$ represent the true trajectory, whereas we know $A_1\cdots A_nB_1\cdots B_n$ from sequence matching. In the analysis, we denote $n_v= \vert \mathcal{V}_{h,l}\vert$ as the cardinality of $\mathcal{V}_{h,l}$ and $n_b$ as the expected number of neighbors for each vertex. We describe map/trajectory property in a rudimentary way by assuming $k_d$ levels of distinguishable discrete headings in $[0, 2\pi)$ and $k_l$ levels of distinguishable discrete road lengths. Each vertex takes a heading value and length value with equal probabilities of $1/k_d$ and $1/k_l$ correspondingly. Generally speaking, we know $n_v \gg k_d \geq n_b$ and $n_v \gg k_l \geq n_b$ for most maps. we have the following lemma.
\begin{Lem}\label{lem:prob_truematch1}	 
The conditional probability that $M=\{\Theta, D\}$ is the true matching sequence given that $Q=\{\Theta_q, D_q\}$ matches $M$ is,
\begin{equation}\label{eq:prob_truematch1}	 
\resizebox{.9 \linewidth}{!} 
{
$P(C_1\cdots C_n|A_1\cdots A_n B_1\cdots B_n)= \frac{(1-\alpha)^2k_dk_l}{n_v}\left[(1-\alpha)^2\frac{k_dk_l}{n_b} \right]^{n-1}$
}
\end{equation}	
\end{Lem}
\begin{proof}	
	Applying the Bayesian equation, we have
	\begin{align}
	\nonumber P(C_1 \cdots C_n &|A_1\cdots A_n B_1 \cdots B_n) = \\
	& \frac{P(A_1\cdots A_n B_1\cdots B_n|C_1\cdots C_n)P(C_1\cdots C_n) }{P(A_1\cdots A_nB_1\cdots B_n)}\label{eq:conditioned_matching}.
	\end{align}
	Indeed $P(A_1\cdots A_n B_1\cdots B_n|C_1\cdots C_n)$ is the conditional probability that a correct matched sequence survives $n$ hypothesis tests in~\eqref{eq:hypo_vertexmatch}. Due to independent measurement noises, we have $P(A_1B_1|C_1)= (1-\alpha)^2$. Besides, these tests are independent due to independent sensor noises, we have
	\begin{equation}\label{eq:Bayesian_format_reverse}
	P(A_1\cdots A_n B_1\cdots B_n|C_1 \cdots C_n) = (1-\alpha)^{2n}.
	\end{equation}
	Joint probability $P(C_1\cdots C_n)$ is actually the unconditional probability of being correct locations. We know $P(C_1)=1/n_v$ given there are $n_v$ possible solutions, and $P(C_2|C_1)=1/n_b$ because there are $n_b$ neighbors of $C_1$. By induction, 
	\begin{equation}\label{eq:Bayesian_bominator}
	P(C_1\cdots C_n) = \frac{1}{n_b^{n-1}}  \frac{1}{n_v}.
	\end{equation}
	 Lastly, each vertex takes a heading value and length value with equal and independent probabilities of $1/k_d$ and $1/k_l$. We have $P(A_kB_k) = \frac{1}{k_dk_l}$ and 
	\begin{equation}\label{eq:Bayesian_denominator}
	P(A_1\cdots A_n B_1\cdots B_n) = \frac{1}{(k_dk_l)^n}. 
	\end{equation}
	Plugging \eqref{eq:Bayesian_format_reverse}, \eqref{eq:Bayesian_bominator}, and \eqref{eq:Bayesian_denominator} into \eqref{eq:conditioned_matching}, we obtain the lemma.
\end{proof}

\begin{Cor} \label{remark:loc_speed}
We have shown in~\cite{cheng2018plam} that the conditional probability that $\Theta$ is the true matching given $\Theta_q$ is 
\begin{equation} \label{eq:prob_truematch}
\resizebox{.85 \linewidth}{!} 
{	
	$P(C_1\cdots C_n|B_1\cdots B_n)= \frac{(1-\alpha)k_d}{n_v}\left[(1-\alpha)\frac{k_d}{n_b} \right]^{n-1}$
}
\end{equation}
Compare \eqref{eq:prob_truematch1} with \eqref{eq:prob_truematch}, we have
\begin{equation}
\dfrac{P(C_1\cdots C_n|A_1\cdots A_n B_1\cdots B_n)}{P(C_1\cdots C_n|B_1\cdots B_n)}=[(1-\alpha)k_l]^{n}
\end{equation}
Since $k_l>\frac{1}{1-\alpha}$ is generally true, localization using both heading and length information $Q=\{\Theta_q,D_q\}$ is faster than using heading $\Theta_q$ only.    
\end{Cor}

According to \eqref{eq:prob_truematch1}, $(1-\alpha)^2\frac{k_dk_l}{n_b}$ determines localization efficiency which is related to both $k_d$ and $k_l$, the spreading of both heading and road length. To better understand how it stands in real world, we analyze map proprieties in the following section. 

\subsubsection{Map Entropy Analysis} \label{ssec:mapentropy}
To provide a measure of variation and spreading in heading and road length, we introduce the Shannon information entropy to measure road heading and length distributions\cite{mohajeri2014evolution}. To minimize the effect of bin size on calculated entropy, we set orientation bin widths to be $5$\textdegree, and 20 meters for road length. Let us denote orientation range set by $\{O_\mathtt{j} \vert \mathtt{j} = 1,2,\cdots,n_\mathtt{j}\}$ and length range set by $\{L_\mathtt{i} \vert \mathtt{i} = 1,2,\cdots,n_\mathtt{i}\}$. We define $n_{ji}=n_jn_i$ and $\rho_\mathtt{ji}$ be the relative frequency that $\theta_i\in O_\mathtt{j}$ and $d_i\in L_\mathtt{i}$.
The joint Shannon entropy in heading and road length is
\begin{align}
H_{\theta,d}(\mathcal{V}_{h,l})
&=-\sum_{\mathtt{j}}\sum_{\mathtt{i}} \rho_{\mathtt{ji}}\log_{n_\mathtt{ji}}\rho_{\mathtt{ji}}. 
\end{align}
By analyzing the entropy of different maps, we predict localization efficiency of our algorithm, which will be shown in Section~\ref{sec:exps}.

\subsection{Location Alignment and Verification Thread}

\begin{figure}[th!]
	\subfigure[]{\includegraphics[width=0.51\linewidth, viewport=16 68 993 720, clip=true]{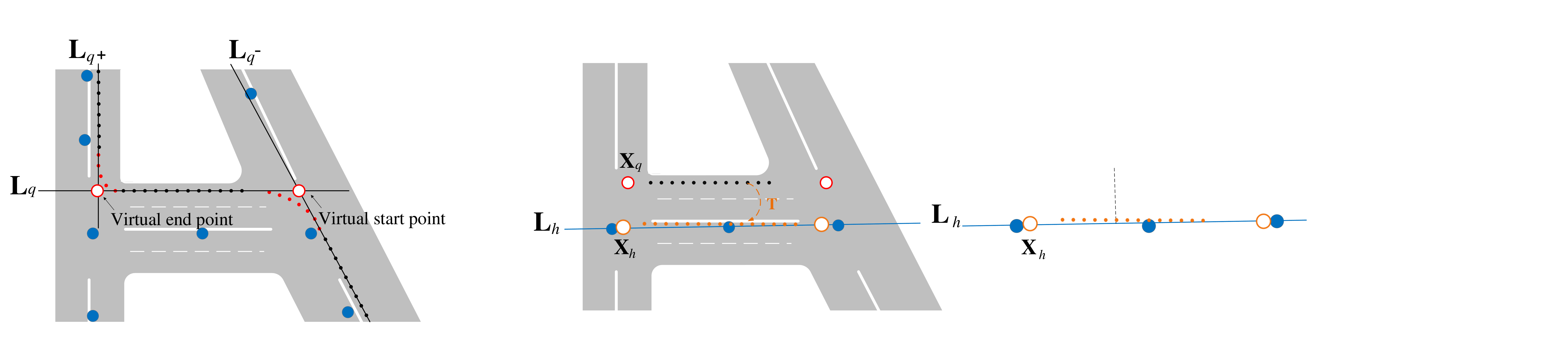}\label{fig:findVpoint}}
	\subfigure[]{\includegraphics[width=0.45\linewidth, viewport=1181 68 2016 662, clip=true]{figures/Road_intersection_v2.pdf}\label{fig:alignVpoint}}
	\caption{Illustration of LAV. The solid small dots represent vehicle trajectory where red points are turn points and black points belong to SSPTE. The roads are shaded gray regions characterizing their width, and GPS waypoints in $\mathcal{M}_p$ are represented in larger blue dots. (a) Virtual starting and end points (i.e. red circles) of an SSPTE. (b) Left: misalignment between $\mathbf{X}_q$ and $\mathbf{X}_h$. It is clear that SSPTM only has three points. Exact point-to-point matching is not appropriate. We fit a line $\mathbf{L}_h$ using SSPTM which is used as reference line for finding the best transformation between SSPTE and SSPTM points.}
	\label{fig:VPoints}
\end{figure}

If the GL thread finds a unique position, we can start LAV thread to continuously report vehicle location. The key is to fix the EKF drift issue using the prior map information. This is achieved by monitoring if the vehicle makes a turn. Once a turn is identified, the straight segment prior to the turn (SSPT) can be extracted. Comparing the SSPT from EKF estimation (SSPTE) to the corresponding SSPT on the map $\mathcal{M}_p$ (SSPTM), we can reset EKF parameters which rectifies the drifting issue. 

Let us define the set of points in SSPTE by
\begin{equation} \label{eq:Xq_SSPTE}
\mathbf{X}_q=\{\mathbf{p}_{\iota}\in\mathbb{R}^2 \vert \iota=1,\cdots,n_q\}
\end{equation}
with each element obtained from EKF $\mathbf{p}^I_{1:2}=[x, y]^\mathsf{T}$ where $\mathbf{p}^I_{1:2}$ is the first and second element of $\mathbf{p}^I$. The distribution of $\mathbf{p}_{\iota}$ is 
%\begin{equation}\label{eq:dist_xj'}
$\mathbf{p}_{\iota}\sim\mathcal{N}(\mu_{\mathbf{p}_{\iota}},\Sigma_{\mathbf{p}_{\iota}})$,
%\end{equation}
where $\mu_{\mathbf{p}_{\iota}}$ is the mean vector and $\Sigma_{\mathbf{p}_{\iota}}$ is the covariance matrix obtained from the EKF. The corresponding GPS SSPTM points are defined by 
\begin{equation} \label{eq:Xh_SSPTM}
\mathbf{X}_h=\{\mathbf{x}_l\vert l =1,\cdots,n_h \}
\end{equation}
and the covariance of GPS points is denoted by $\Sigma_g=diag(\sigma^2_g,\sigma^2_g)$ as mentioned in Section~\ref{ssec:GraphConstruction}. Thus we have
%\begin{equation}\label{eq:dist_xl}
$\mathbf{x}_l\sim\mathcal{N}(\mu_{\mathbf{x}_l},\Sigma_{g})$.
%\end{equation}

%%% Artificial start/end point
\subsubsection{Virtual Starting-Point and End-Point Estimation}
However, SSPTE points do not necessary follow SSPTM as shown in Fig.~\ref{fig:findVpoint}. This is because we do not know which lane the vehicle is driving in and the map may not provide lane-level waypoint accuracy. Fig.~\ref{fig:findVpoint} also shows the effect of vehicle turn radius which makes the length of SSPTE shorter than that of the corresponding SSPTM. To address the problem, we estimate virtual starting and end points for an SSPTE.

We find the virtual starting and end points by computing line intersection of two consecutive SSPTE segments. With the current segment positions $\mathbf{X}_q$, we denote the set of points from previous and next SSPTE segments by $\mathbf{X}_{q^-}$ and $\mathbf{X}_{q^+}$, respectively. Applying line fitting to $\mathbf{X}_q$, $\mathbf{X}_{q^-}$, and $\mathbf{X}_{q^+}$, we obtain three 2D lines $\mathbf{L}_{q}$, $\mathbf{L}_{q^-}$, and $\mathbf{L}_{q^+}$, respectively.  We parameterize each line by two reference points.  Thus we denote $\mathbf{L}_{q}=[\mathbf{a}_q^\mathsf{T},\mathbf{b}_q^\mathsf{T}]^\mathsf{T}$, $\mathbf{L}_{q^-}=[\mathbf{a}_{q^-}^\mathsf{T},\mathbf{b}_{q^-}^\mathsf{T}]^\mathsf{T}$, and 
$\mathbf{L}_{q^+}=[\mathbf{a}_{q^+}^\mathsf{T},\mathbf{b}_{q+}^\mathsf{T}]^\mathsf{T}$. Also, the line direction vectors are $\mathbf{v}_q = \mathbf{b}_{q} - \mathbf{a}_{q}$, $\mathbf{v}_{q^+} = \mathbf{b}_{q^+} - \mathbf{a}_{q^+}$, and $\mathbf{v}_{q^-} = \mathbf{b}_{q^-} -\mathbf{a}_{q^-}$. Finding the intersection between $\mathbf{L}_{q}$ and $\mathbf{L}_{q^-}$ allows us to obtain the virtual starting point. We denote the virtual starting point of $\mathbf{X}_q$ by $\mathbf{p}_{s}$.  
\begin{equation} \label{eq:virtualStartPt}
\mathbf{p}_{s} = \mathbf{a}_q - \dfrac{\mathbf{v}^{\perp}_{q^{-}}.(\mathbf{a}_q-\mathbf{a}^{-}_q)}{\mathbf{v}^{\perp}_{q^{-}}.\mathbf{v}_q}\mathbf{v}_q,
\end{equation}
where $\cdot$ is dot product and $\mathbf{v}^{\perp}_{q^{-}}$ is the perp operator of $\mathbf{v}_{q^{-}}$. Similarly, the intersection between $\mathbf{L}_{q}$ and $\mathbf{L}_{q^+}$ gives us the virtual end point $\mathbf{p}_{e}$. We have 
\begin{equation} \label{eq:virtualEndPt}
\mathbf{p}_{e} = \mathbf{a}_q - \dfrac{\mathbf{v}^{\perp}_{q^{+}}.(\mathbf{a}_q-\mathbf{a}^{+}_q)}{\mathbf{v}^{\perp}_{q^{+}}.\mathbf{v}_q}\mathbf{v}_q,
\end{equation}
where $\mathbf{v}^{\perp}_{q^{+}}$ is the perp operator of $\mathbf{v}_{q^{+}}$.
When SSPTE is connected with an curve segment (e.g. caused by vehicle turn), we add $\mathbf{p}_{s}$ and $\mathbf{p}_{e}$ to $\mathbf{X}_q$ to help alignment process. $\mathbf{p}_{s}$ and $\mathbf{p}_{e}$ become the first and the last points in $\mathbf{X}_q$, respectively.

\subsubsection{Location Alignment and Verification}
With augmented $\mathbf{X}_q$, we can match $\mathbf{X}_q$ to $\mathbf{X}_h$ to rectify drifting issue by finding the transformation $\mathbf{T}$ between them (see Fig.~\ref{fig:loc_align}). Here $\mathbf{T}$ is 3-DoF rigid body transformation represented by a 2x2 rotation matrix $\mathbf{R}$, and a 2x1 translation vector $\mathbf{t}$, 
\begin{equation} \label{eq:rigid_transform}
\mathbf{T}(\mathbf{x}):=\mathbf{R}\mathbf{x}+\mathbf{t},
\end{equation}
where $\mathbf{x}$ is a 2D point. $\mathbf{X}_q$ usually contains significantly more entries than that of $\mathbf{X}_h$ due to its higher sampling frequency ($n_q \gg n_h$). Directly matching two point sets is not the best solution. Instead, we fit a line through points in $\mathbf{X}_h$ and minimizing the distance of all points in $\mathbf{X}_q$ to this line (Fig.~\ref{fig:alignVpoint}). 

\begin{figure}[ht!]
	\subfigure[$n=4$]{\includegraphics[width=0.32\linewidth, viewport=110 243 480 540, clip=true]{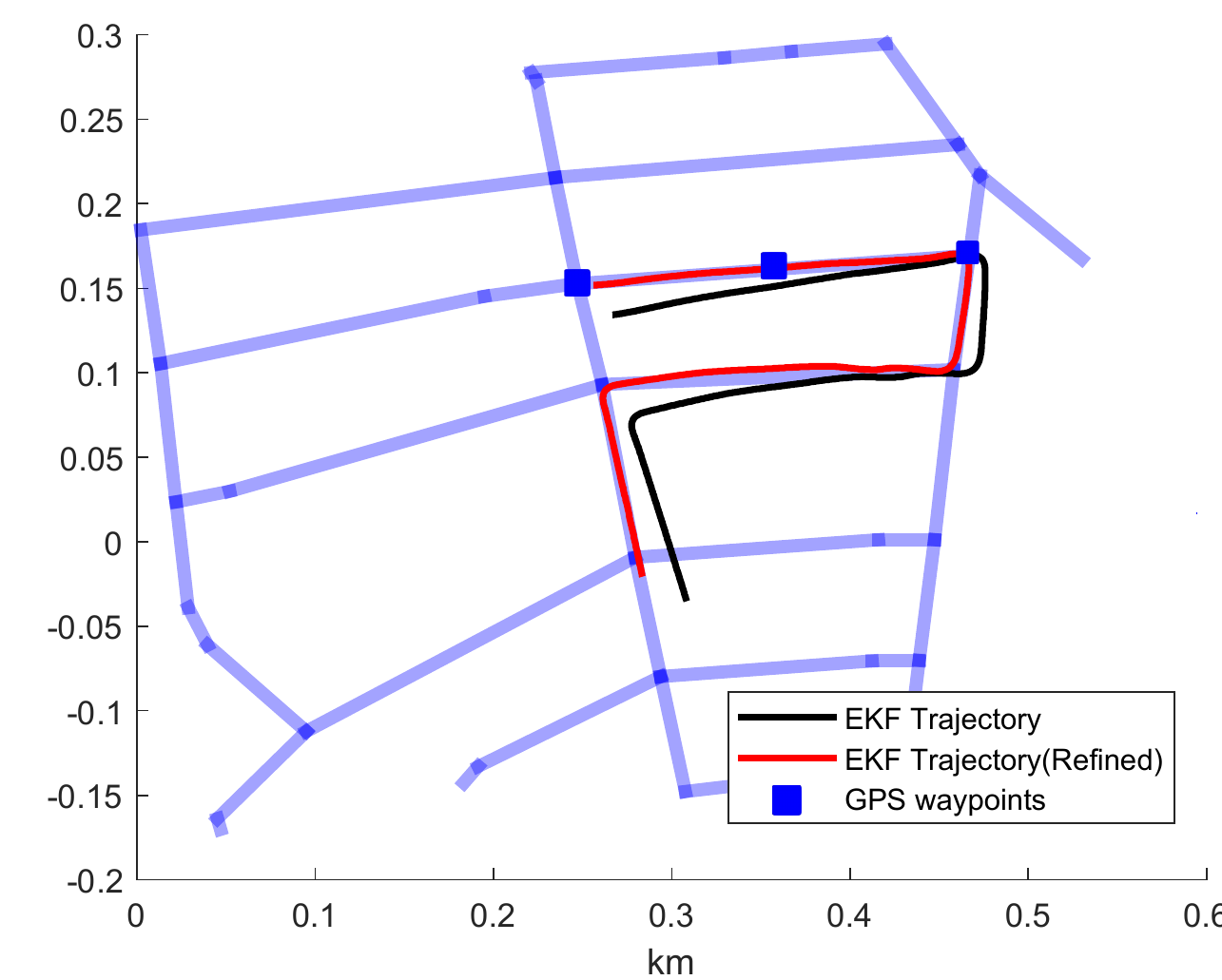}}
	\subfigure[$n=5$]{\includegraphics[width=0.32\linewidth, viewport=110 243 480 540, clip=true]{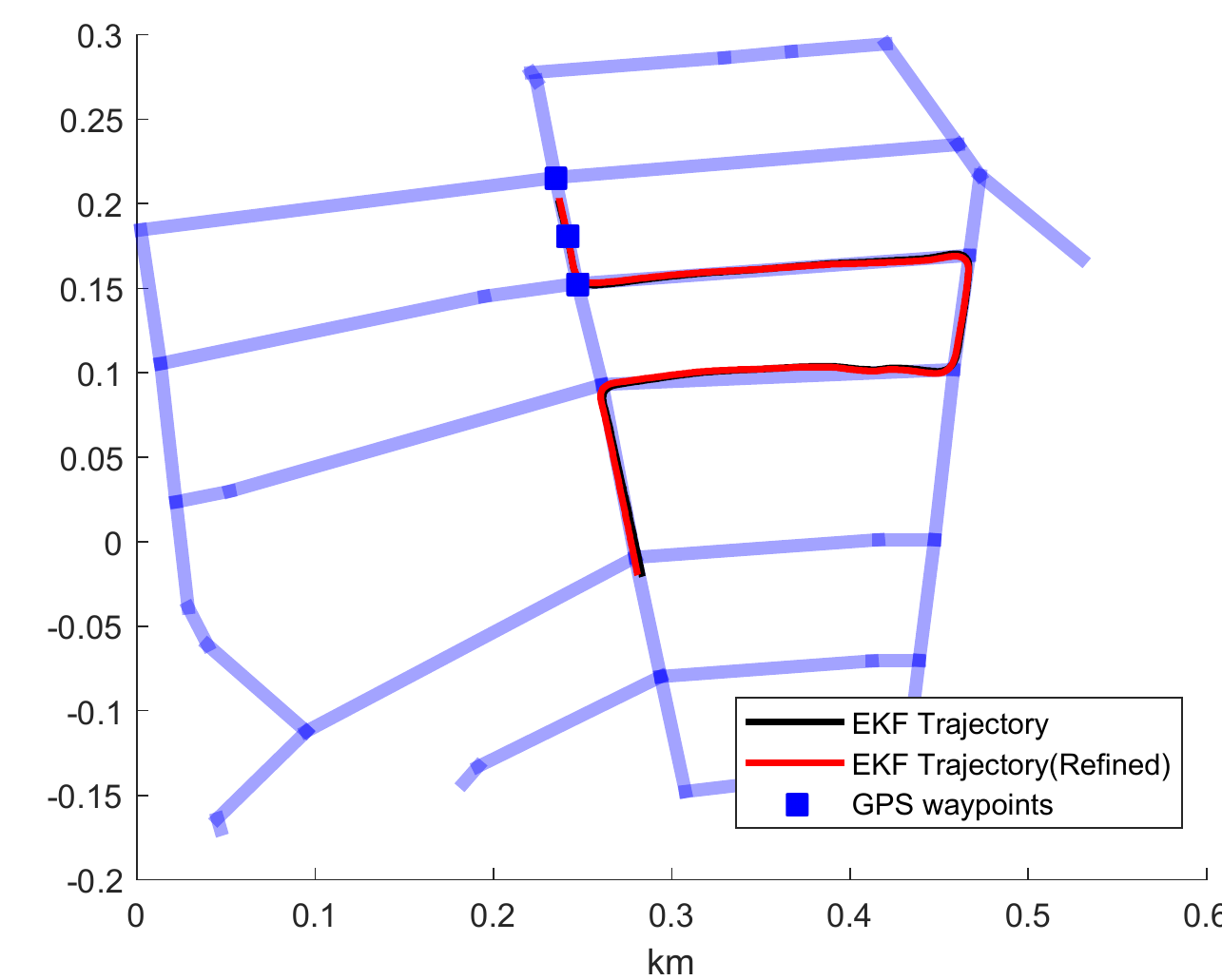}}
	\subfigure[$n=6$]{\includegraphics[width=0.33\linewidth, viewport=110 243 480 540, clip=true]{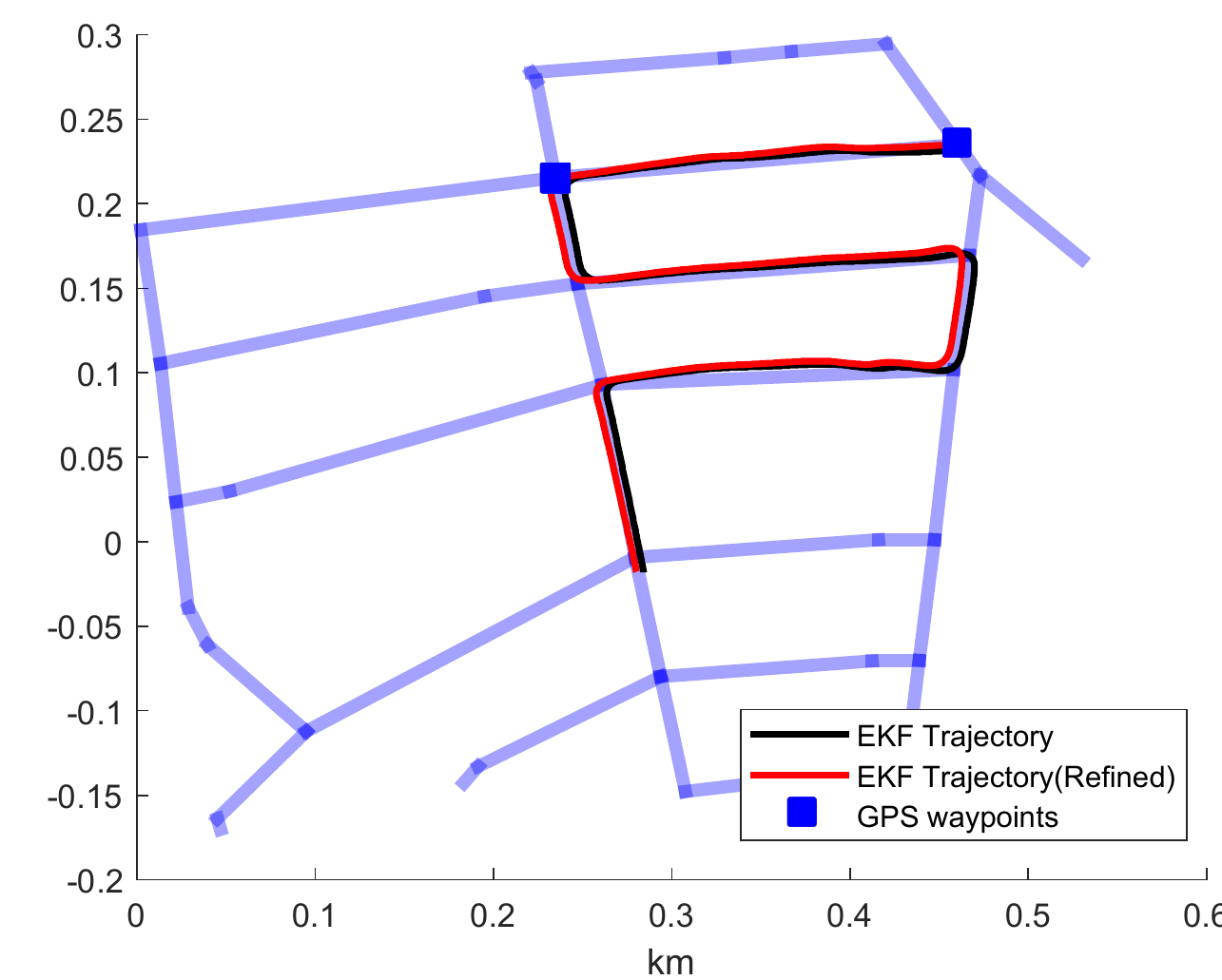}}
	\caption{An example of location alignment and verification that keeps drifting under control where $n$ is the number of long straight segments for the vehicle. The unaligned trajectory is shown in black, the aligned trajectory is shown in red, and GPS waypoints are shown in dark blue square.}
	\label{fig:loc_align}
\end{figure}

Let us denote $\mathbf{L}_{h}=[\mathbf{a}_h^\mathsf{T},\mathbf{b}_h^\mathsf{T}]^\mathsf{T}$ where $\mathbf{a}_h$ and $\mathbf{b}_h$ are two reference points on the line. For every point $\mathbf{p}_j$ in $\mathbf{X}_q$, the point after transformation is denoted by $\mathbf{T}(\mathbf{p}_{\iota})$. The point-to-line distance between $\mathbf{T}(\mathbf{p}_{\iota})$ and $\mathbf{L}_h$ is defined as
\begin{equation}
d_{\perp}(\mathbf{T}(\mathbf{p}_{\iota}),\mathbf{L}_h) = \dfrac{||(\mathbf{a}_h-\mathbf{T}(\mathbf{p}_{\iota})\times(\mathbf{a}_h-\mathbf{b}_h)||}{||\mathbf{a}_h-\mathbf{b}_h||},
\end{equation}
where `$\times$' is the cross product and $||\cdot||$ is the $L^2$ norm.
\begin{comment}
Let us denote $\mathbf{L}_h:a_hx+b_hy+c_h=0 $ as the fitted line given $\mathbf{X}_h$. For every point $\mathbf{p}_j$ in $\mathbf{X}_q$, the point after transformation is denoted by $\mathbf{T}(\mathbf{p}_j)=[x_{p_j}, y_{p_j}]^{\mathsf{T}}$. The point-to-line distance between $\mathbf{T}(\mathbf{p}_j)$ and $\mathbf{L}_h$ is defined as
\begin{equation}
d_{\perp}(\mathbf{T}(\mathbf{p}_j),\mathbf{L}_h) = \dfrac{|a_hx_{p_j}+b_hy_{p_j}+c_h|}{\sqrt{a_h^2+b_h^2}},
\end{equation}
where $d_{\perp}(\cdot,\cdot)$ denotes the perpendicular distance from a point to a line in 2-D, as illustrated in Fig.~\ref{fig:alignVpoint}.
\end{comment}
We define the cost function $\mathbf{C}_\mathbf{T}$ by
\begin{equation}
\mathbf{C}_\mathbf{T}=
\begin{bmatrix}
d_{\perp}(\mathbf{T}(\mathbf{p}_s),\mathbf{L}_h)\\
d_{\perp}(\mathbf{T}(\mathbf{p}_1),\mathbf{L}_h)\\
\vdots\\
d_{\perp}(\mathbf{T}(\mathbf{p}_{n_q}),\mathbf{L}_h)\\
d_{\perp}(\mathbf{T}(\mathbf{p}_e),\mathbf{L}_h)\\
\end{bmatrix},
\end{equation}  and formulate the following optimization problem
\begin{equation} \label{eq:opt_traj_unconstrained}
\argmin{\mathbf{T}}\mathbf{C}_\mathbf{T}^{\mathsf{T}}\Sigma_C^{-1}\mathbf{C}_\mathbf{T}+ \lambda||\mathbf{T}(\mathbf{p}_s)- \mathbf{x}_1||+\lambda||\mathbf{T}(\mathbf{p}_e)-\mathbf{x}_{n_h}||,
\end{equation}
where $\Sigma_C=diag(\sigma^2_{d_{\perp},\mathbf{p}_s},\cdots,\sigma^2_{d_{\perp},\mathbf{p}_e})$, $\beta$ is a nonnegative weight, and $\mathbf{x}_1$ and $\mathbf{x}_{n_h}$ are the first and the last entries in \eqref{eq:Xh_SSPTM}, respectively. $\sigma^2_{d_{\perp},\mathbf{p}_\iota}$ is obtained using error propagation. In detail, let $d_{\perp}(\mathbf{T}(\mathbf{p}_\iota),\mathbf{L}_h)= f_d(\mathbf{p}_\iota,\mathbf{L}_h)$ and $\xi=[\mathbf{p}_s^\mathsf{T},\mathbf{L}_h^\mathsf{T}]^\mathsf{T}$, we have $\sigma^2_{d_{\perp},\mathbf{p}_\iota} = J_d\Sigma_d J_d^\mathsf{T}$, where $J_d = \frac{\partial f_d}{\partial\xi}$ and $\Sigma_d=diag(\Sigma_{\mathbf{p}_\iota}, \Sigma_{\mathbf{L}_h})$ because $\mathbf{p}_\iota$ is independent of $\mathbf{L}_h$ which comes from $\mathbf{X}_h$. Define $\mathbf{L}_h=f_L({\mathbf{X}_h})$, we have $\Sigma_{\mathbf{L}_h}=J_L\Sigma_{{X}_h}J_L^{\mathsf{T}}$ where  $J_L=\frac{\partial{f_L}}{\partial{\mathbf{X}_h}}$ and $\Sigma_{{X}_h}=diag(\Sigma_g,~\cdots, \Sigma_g)$. 
%where $\Sigma_C=diag(\Sigma_g,\cdots,\Sigma_g)$ and $\Sigma_g=diag(\sigma_g^2,\sigma_g^2)$, $\lambda$ is a nonnegative weight, and $\mathbf{x}_1$ and $\mathbf{x}_{n_h}$ are the first and the last entries in $\mathbf{X}_h$, respectively. Here we neglect the covariance from EKF in $\Sigma_C$ since in general it is much smaller than $\Sigma_{g}$ in short distance. This is due to the fact that the wheel encoder regulates the distance drift of EKF (see (\ref{eq:obs_obdvel})) and compass readings regulate angular drift (see (\ref{eq:obs_phi})). 
The second and third terms are soft constraints due to potential alignment errors. To solve~\eqref{eq:opt_traj_unconstrained}, we start with a small positive weight for $\lambda$ and apply a nonlinear optimization solver, e.g. Levenberg-Marquardt algorithm. Initially, we set $\mathbf{R}=\mathbf{I}_{2\times2}$, and $\mathbf{t}$ from the result of the global location obtained from Section~\ref{ssec:GlobalLoc}. For each turn, we use previous solution as the initial solution and increase $\lambda$  gradually until the change in solution is negligible.

Now we have optimized $\mathbf{T}$ and we denote the aligned locations by $\mathbf{\hat{X}}_q=\mathbf{T}(\mathbf{X}_q)$. We need to verify if the matching result is reliable by performing hypothesis testing. We have two hypotheses:
\begin{eqnarray}
& \mathbf{H_{0}}:& \mbox{$\mathbf{X}_h$ and $\mathbf{\hat{X}}_q$ are from the same distribution,}\nonumber\\
& \mathbf{H_{1}}:& \mbox{~otherwise.}  \label{eq:hypo_vertexmatch2}
\end{eqnarray}
We set the significance level by $\alpha$ and reject $H_0$ if the statistic is less than $\alpha$. Note $H_0$ is examined by the Mahalanobis distance $\mathbf{C}_\mathbf{T}^{\mathsf{T}}\Sigma_C^{-1}\mathbf{C}_\mathbf{T}$ which follows a $\chi^2$ distribution with $2(n_q + 2)$ DoFs. Thus we reject $H_0$ if 
$$\mathbf{C}_\mathbf{T}^{\mathsf{T}}\Sigma_C^{-1}\mathbf{C}_\mathbf{T}>\chi^2_{2(n_q+2)}(\alpha).$$
Correspondingly, we set localization status indicator variable $I_G$ values by
\begin{equation}  \label{eq:Ig_reset}
I_G =
\begin{cases}
0,\mbox{~$H_{0}$ is rejected},\\
1,\mbox{~otherwise}.
\end{cases}
\end{equation}	
If $I_G=1$, we accept $\mathbf{T}$ and use the aligned trajectory $\hat{\mathbf{X}}_q:=\mathbf{T}(\mathbf{X}_q)$ which is used to reset the EKF states (Fig.~\ref{fig:SystemDiagram}). After LAV execution, we keep acquiring the vehicle locations EKF $\mathbf{p}^{I}_{1:2}$ until next turn. When turn is detected and $I_G=1$, we execute LAV thread repeatedly. If $I_G=0$, it means that we cannot find the position and we lose the global position. Thus we terminate the LAV thread and start the GL thread again.

\subsubsection{SSF Estimation}\label{ssec:ScaleAdjustment} 
To further reduce drift in the dead-reckoning process, we consider SSF in the EKF-based trajectory estimation. There are two sources of biases: systematic and non-systematic biases from wheel encoder 
inputs~\cite{borenstein1995correction}. The systematic error can be caused by tire radius error such as inflation level, tire wear, gear ratio, etc. Non-systematic error comes from wheel slippage on road. To compensate for those errors, we introduce scale and slip factor  $s_{ssf}$ in \eqref{eq:obs_scale}. To compute $s_{ssf}$, we need the travel length for each vertex on HLG for both query data and map data. We obtain the travel length $d_q$ on the query data using the virtual starting/end points $\mathbf{p}_e$ and $\mathbf{p}_s$ in \eqref{eq:virtualStartPt} and \eqref{eq:virtualEndPt}. That is,
\begin{equation}
d_q = ||\mathbf{p}_e - \mathbf{p}_s||
\end{equation}
According to \eqref{eq:Xh_SSPTM}, the corresponding travel length on the map is denoted by $d:=||\mathbf{x}_{n_h} - \mathbf{x}_{1}||$. Assuming GL thread ends at the $n$-th turn, for $k=(n+1), \cdots, n'$ we estimate $s_{ssf}$ by computing the ratio of accumulated
length $d_{q,k}$ and $d_{k}$:
\begin{equation}
\label{eq:scale_est}
s_{ssf} = {\sum\limits_{k = n+1}^{n'}d_{k}}\bigg/{\sum\limits_{k = n+1}^{n'} d_{q,k}}.
\end{equation}

We then model the variance of $s_{ssf}$ to be used in the EKF measurement variance in Section~\ref{ssec:ETE}. It is not accurate to set a constant variance value for $s_{ssf}$, since at the beginning traveling length is short and thus $s_e$ has larger variance. As the traveling length increases, the variance of $s_{ssf}$ ought to decrease. Denote the variance of $s_{ssf}$ by $\sigma^2_{s_{ssf}}$, we derive the following Lemma.
\begin{Lem}
	\label{lem:scale}
	The variance of scale and slip factor $s_{ssf}$ is 
	\begin{equation} 
	\sigma^2_{s_{ssf}} =\frac{1}{L^2_q}(2n_s\sigma^2_{g} + \frac{L^2_g}{L^2_q}\sum_{k = n+1}^{n'}\sigma^2_{d_q,k}).
	\label{eq:scale_var}
	\end{equation}
\end{Lem}

\begin{proof}
	First, we write $s_{ssf}$ as function of measurements from $d_k$ and $d_{q,k}$ according to \eqref{eq:scale_est}. That is, $s_{ssf} =f_s(d_{n+1},\cdots,d_{n'},d_{q,n+1},\cdots,d_{q,n'})$. We know the variance of $d_{k}$ is $\sigma^2_{d_k}=2\sigma^2_{g}$ from \eqref{eq:dist_di} and the variance of $d_{q,k}$ is $\sigma^2_{d_q,k}$ which is defined in Section~\ref{ssec:TurnDetection_QuerySeq}. Let us define $L_q = \sum_{k = n+1}^{n'}d_{q,k}$, $L_g = \sum_{k = n+1}^{n'}d_{k}$, and $n_s=n'-n$. 
	Through forward error propagation,  
	\begin{equation}
	\label{eq:scale_var1}
	\sigma^2_{s_{ssf}} = J_s\Sigma_sJ_s^{\mathtt{T}},
	\end{equation}
	where $\Sigma_s=diag(2\sigma^2_{g},\cdots,2\sigma^2_{g},\sigma^2_{d_{q,n+1}}\cdots\sigma^2_{d_{q,n'}})$  and $J_s$ is 
	\begin{align} \label{eq:scale_Js}
	\notag J_s &= [\frac{\partial f_s}{\partial d_{n+1}},\cdots,\frac{\partial f_s}{\partial d_{n'}}, \frac{\partial f_s}{\partial d_{q,n+1}},\cdots,\frac{\partial f_s}{\partial d_{q,n'}}]\\
	& = [\dfrac{1}{L_q}\cdots,\dfrac{1}{L_q},\dfrac{-L_g}{L^2_q},\cdots,\dfrac{-L_g}{L^2_q}].
	\end{align}
	Plug~\eqref{eq:scale_Js} into~\eqref{eq:scale_var1}, we have
	\begin{align} \label{eq:scale_var2}
	\notag \sigma^2_{s_{ssf}} = J_s\Sigma_sJ_s^{\mathtt{T}}  &= 2n_s\frac{\sigma^2_{g}}{L^2_q} + \sum_{k = n+1}^{n'}\sigma^2_{d_q,k}\frac{L^2_g}{L^4_q} \\
	&=\frac{1}{L^2_q}(2n_s\sigma^2_{g} + \frac{L^2_g}{L^2_q}\sum_{k = n+1}^{n'}\sigma^2_{d_q,k}).
	\end{align}	
	\begin{comment}
	Let us take a close look at \eqref{eq:scale_var2}. 
	Since GPS uncertainty is often much larger than the EKF uncertainty in short segments, that is $\sigma^2_{d_k}=2\sigma^2_{g}\gg \sigma^2_{d_{q,k}}$, we neglect the second term in \eqref{eq:scale_var2}. 

	Also, we have $L_q\approx L_g$ because the travel lengths should be similar. Therefore, we can approximate $\sigma^2_{s_{ssf}}$ as 
	\begin{equation}
	\sigma^2_{s_{ssf}} \approx \frac{(2n_s\sigma^2_{g} )}{L^2_q}= \frac{2(n'-n)\sigma^2_{g}}{(\sum\limits_{k = n+1}^{n'}d_{q,k})^2}.
	\end{equation}
	\end{comment}
\end{proof}	
\begin{remark}
	Let us take a close look at \eqref{eq:scale_var2}. We have $L_q\approx L_g$ because the estimated travel length should be similar to the corresponding path in map. Therefore, we can approximate $\sigma^2_{s_{ssf}}$ as 
	\begin{align} 
	\notag \sigma^2_{s_{ssf}} = J_s\Sigma_sJ_s^{\mathtt{T}}  
	&=\frac{1}{L^2_q}(2n_s\sigma^2_{g} + \sum_{k = n+1}^{n'}\sigma^2_{d_q,k}).
	\end{align}	
	Thus we show that $\sigma^2_{s_{ssf}}$ decrease as $L_q = \sum\limits_{k = n+1}^{n'}d_{q,k}$ increases. As time goes, we have longer travel length and the estimation of $s_{ssf}$ becomes more accurate. Using the accumulated travel length to adjust SSF is suitable to compensate systematic biases. If the traveling length is long and systematic biases are compensated, setting a sliding window for accumulated distance can be used to detect non-systematic biases that varies through traveling.
\end{remark}

The resulting $s_{ssf}$ and $\sigma^2_{s_{ssf}}$ are fed into the EKF in Section~\ref{ssec:ETE}. This completes our overall method. %To observe how it works in real data, we show experiment results of scale state estimated in EKF in Section~\ref{sec:exps}.

%%%%%%%%%%%%%%%%%%%%%%%%%%%%%%%%%%%%%%%%%%%%%%%%%%%%%%%%%%%%%%%%%%%%%%%%%%%%%%%%
\section{Experiments} \label{sec:exps}
We have implemented the proposed GBPL method using MATLAB and validated the algorithm in both simulation and physical experiments. We first validate the proposed global localization approach. Second, we test the LAV performance. 

For physical experiments, we evaluate our approach on three maps with seven outdoor data sets, as described below. We obtain the corresponding three maps from OSM:
\begin{itemize}
	\item CSMap  : College Station, Texas, U.S.
	\item KITTI00Map: Karlsruhe, Germany, and
	\item KITTI05Map: Karlsruhe, Germany.
\end{itemize}
Map information including map size, total length of drivable roads, HLG entropy, and $\#$nodes in HLG is shown in the first four columns of Tab.~\ref{tab:MapsInfo}. 

The seven query sequences are three self-collected CSData sequences and four KITTI sequences: 
\begin{itemize}
	\item CSData: We record IMU readings at 400Hz and compass readings at 50Hz using a Google Pixel phone mounted on a passenger car. Also, we read the vehicle speed at 46.6Hz sampling frequency in average using a Panda OBD-II Dongle which provides the velocity feedback from vehicle wheel encoder. We have collected three sequences: CS-1, CS-2 and CS-3.
	\item KITTI: We use the KITTI GPS/IMU dataset~\cite{geiger2012KITTI} which contains synchronized IMU readings from its inertial navigation system (INS) as inputs. We only use the GPS readings to synthesize compass readings to test our algorithm since the data sets do not provide compass readings. We have four sequences: KITTI00-1, KITTI00-2, KITTI05-1, and KITTI05-2.
\end{itemize}

\begin{table}[b!]
	\caption{Map info. and $\#$straight segments $n$ for localization} 
	\label{tab:MapsInfo}
	\resizebox{\textwidth}{!}{ 
		\begin{tabular}{cccccccc|c}
			\hline
			Maps               & Size~($km^2$) & Drivable road~($km$) &Entropy &\#nodes & $n$(PLAM) & $n$ (GBPL)\\
			\hline
		    CSMap              & 3.24  &52.7   &0.724  & 483     &  9,5,6 &\textbf{3,3,2} \\   %seq:18,19,2
			%	KITTI00               & 4.75  &44.2   &0.972   & 583     & 10,5   &\textbf{4,3} \\
			KITTI00Map         & 4.75  &44.2   &0.877   & 583     & 10,5   &\textbf{4,3} \\  % entropy (dist+angle)
			%	KITTI05               & 3.24  &43.7   &0.950   & 548     &  4,5   &\textbf{3,4} \\ 
			KITTI05Map         & 3.24  &43.7   &0.797   & 548     &  4,5   &\textbf{3,4} \\  % entropy (dist+angle)
		
			%	CSMap3                & 5.75  &92.43  &0.994   & 1102    &  \\
			\hline                                             
		\end{tabular}
	}
\end{table}

\subsection{Global Localization Test}

\subsubsection{Evaluation Metrics and Methods Tested} It is worth noting that the speed of methods are characterized by $n$, number of straight segments in the query. Since computation speed is not a concern, we are more interested in how many inputs it takes to localize the vehicle. Therefore, $n$ is a good metric for this.  For a given $n$, the algorithms may provide multiple solutions if there is many similar routes in the map. If the number of solutions is one, then the vehicle is uniquely localized. The number of solutions is also an important measure for algorithm efficiency. Two algorithms are compared in our experiments:
\begin{itemize}
	\item GBPL: Current method that uses both heading and length information of straight segments.
	\item PLAM: The counterpart method using heading only~\cite{cheng2018plam}. 
\end{itemize}

\subsubsection{Map Entropy Evaluation} \label{ssec:toymaps}
Map entropy describes how much the heading and distance distribution spread out in a given map. Higher entropy means distributions are more spread out and hence it is easier for the vehicle to localize itself, as proved in Lem.~\ref{lem:prob_truematch1}. Therefore, we want to find out what are map entropy range of real cities and use the range to test our GBPL. As shown in Fig.~\ref{fig:city_entropy}, we calculate map entropy distributions of 100 cites based on the data from \cite{boeing2018entropyexp}. For comparison, the normalized sum of heading entropy and length entropy are in orange bars, and the heading entropy are in blue bars. For each city, the sum of heading entropy and length entropy is the upper bound of the joint entropy. We generate histogram plots for entropy distribution in Fig.~\ref{fig:city_anglehist} and Fig.~\ref{fig:city_hist}. As shown in Fig.~\ref{fig:city_hist}, 95 cities have entropy values higher than 0.70 and the lowest entropy is around 0.6. This determines that entropy range of maps that we will use to test our algorithm is from 0.60 to 0.99.

To better understand the relationship among HLG entropy, $n$, and the number of solutions, we simulate 40 maps with joint entropy of heading and length ranging from 0.60 to 0.99. Building on the simulation in~\cite{cheng2018plam}, we expand it from Heading Graph to HLG in this work. For completeness, we repeat information about experimental settings here. The simulated maps are with a fixed graph structure, and we increase the entropy level in both heading and length by perturbing selected road intersection positions. For each map, we generate 20 query sequence samples with $n=1, \cdots, 20$ and the uncertainties of orientation and length are considered by setting $\sigma_{\theta_{q,k}}=5^\circ$,  $\sigma_{d_{q,k}}=\sqrt{2}\sigma_g$, and $\sigma_g=5$ meters.  We compute the number of solutions by averaging the results of 20 sequences for each map. The simulation result is shown in Fig.~\ref{fig:Entropy_exp_hl} and we adapt Fig.~\ref{fig:Entropy_exp_h} from~\cite{cheng2018plam} for comparison. 

For PLAM which uses heading only (Fig.~\ref{fig:Entropy_exp_h}), the vehicle can be localized with $n \leq 10$ if the entropy in orientation is above 0.9\cite{cheng2018plam}. Under GBPL, the vehicle can be localized with $n \leq 7$ even if the heading/length entropy is 0.6. It is worth noting that lower entropy means less spreading of heading and segment length and road network is closer to be a rectilinear grid and hence it is more challenging to localize a vehicle in such settings. GBPL appears to be more robust to low map entropy than PLAM. 

Fig.~\ref{fig:Entropy_exp_h} and Fig.~\ref{fig:Entropy_exp_hl} show the number of solutions with regard to $n$ values and different HLG entropy values. We fix the entropy as 0.87 and $n=3$ in Figs.~\ref{fig:obs_sol} and \ref{fig:entro_sol}, respectively to observe how quickly the number of solutions decreases in each setting. It shows the $\#$solutions decreases more rapidly in GBLP than that of PLAM using heading only. This result is consistent with Cor.~\ref{remark:loc_speed}.

\begin{figure}[ht!]
	\subfigure[]{\includegraphics[width=0.95\linewidth, viewport=6 252 612 536, clip=true]{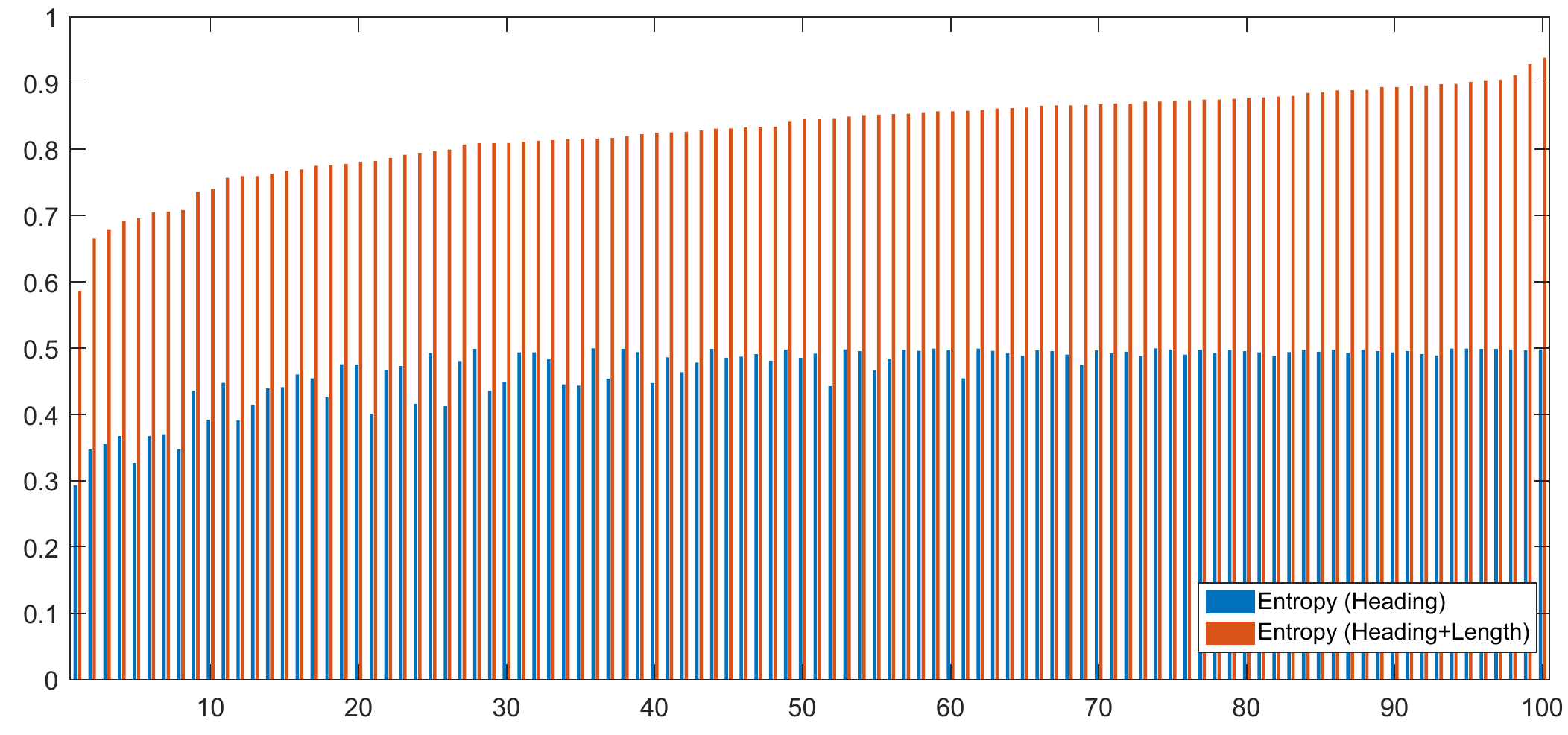}\label{fig:city_entropy}}
	\subfigure[]{\includegraphics[width=0.47\linewidth, viewport=110 240 500 540, clip=true]{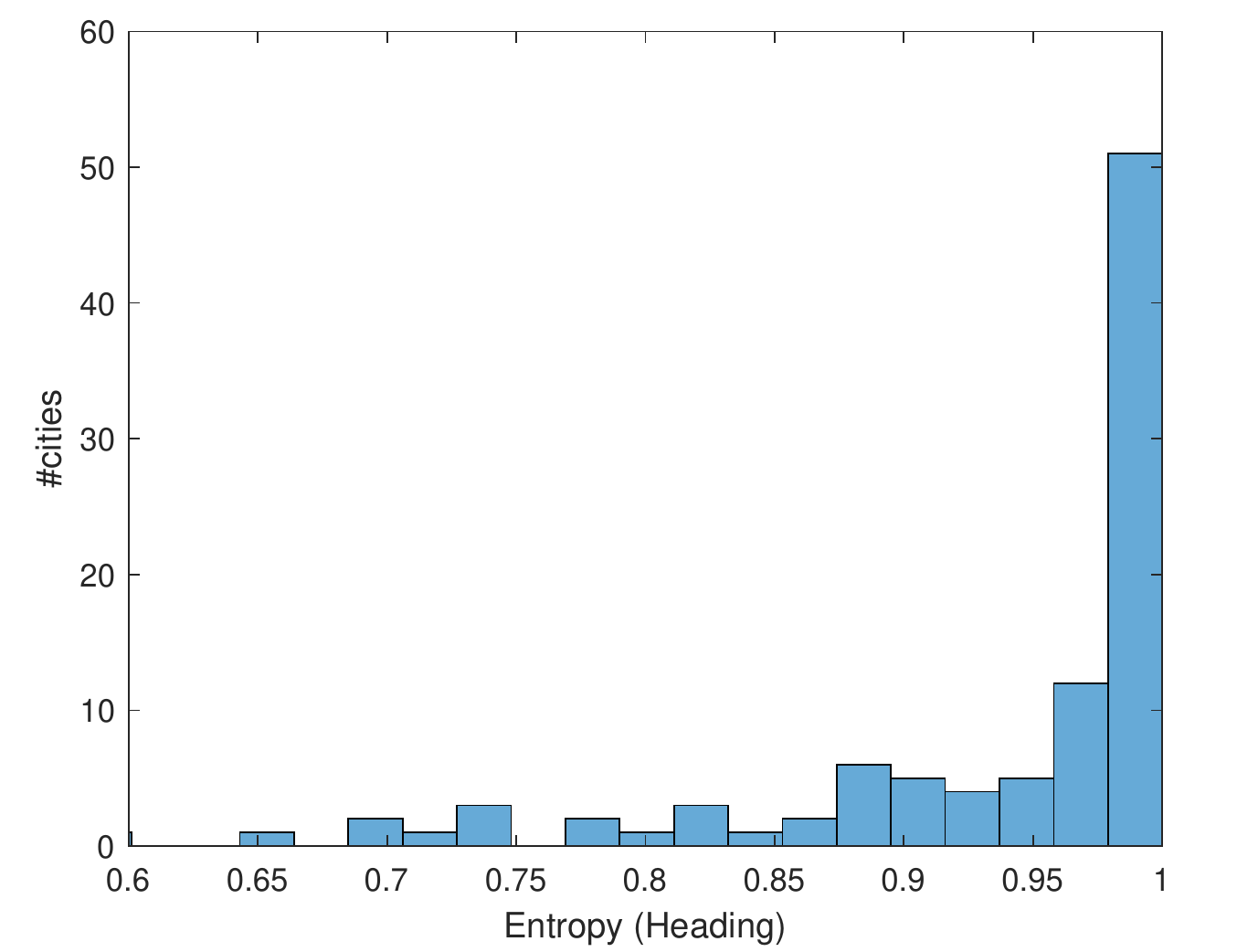}\label{fig:city_anglehist}}
	\subfigure[]{\includegraphics[width=0.47\linewidth, viewport=110 243 500 540, clip=true]{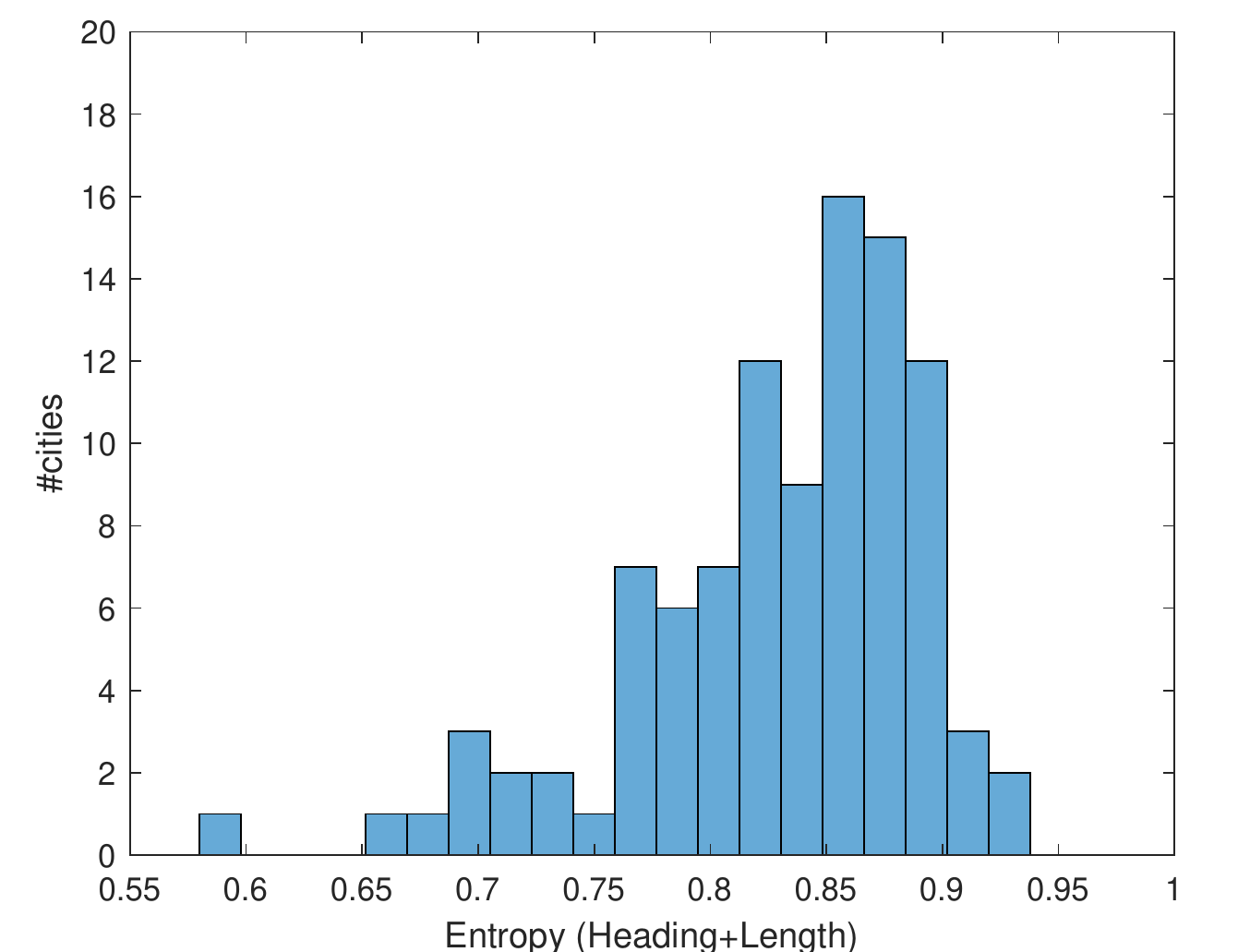}\label{fig:city_hist}}
	\subfigure[]{\includegraphics[width=0.47\linewidth, viewport=110 243 500 540, clip=true]{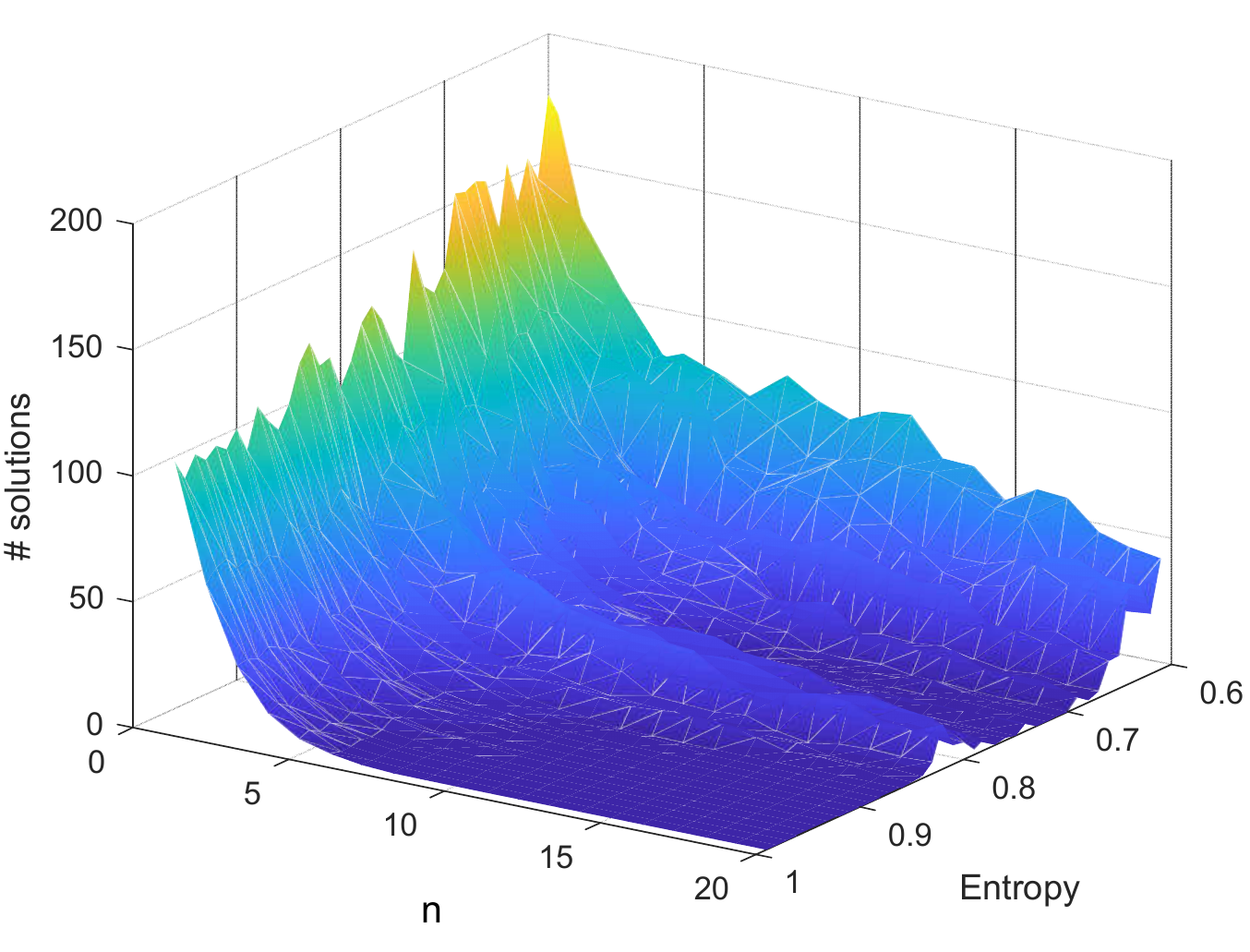}\label{fig:Entropy_exp_h}}
	\subfigure[]{\includegraphics[width=0.47\linewidth, viewport=110 243 500 540, clip=true]{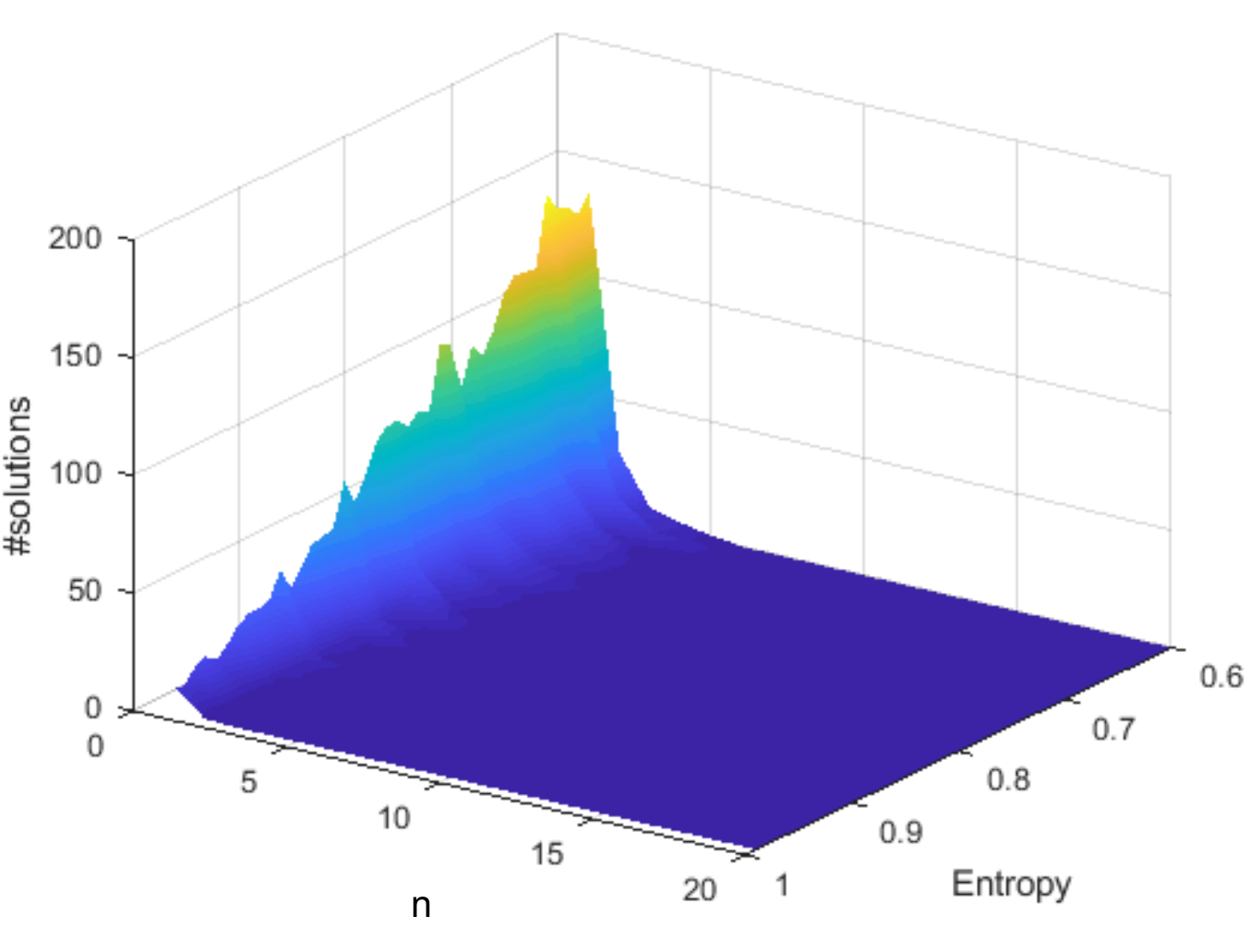}\label{fig:Entropy_exp_hl}}
	\subfigure[]{\includegraphics[width=0.47\linewidth, viewport=64 105 720 552, clip=true]{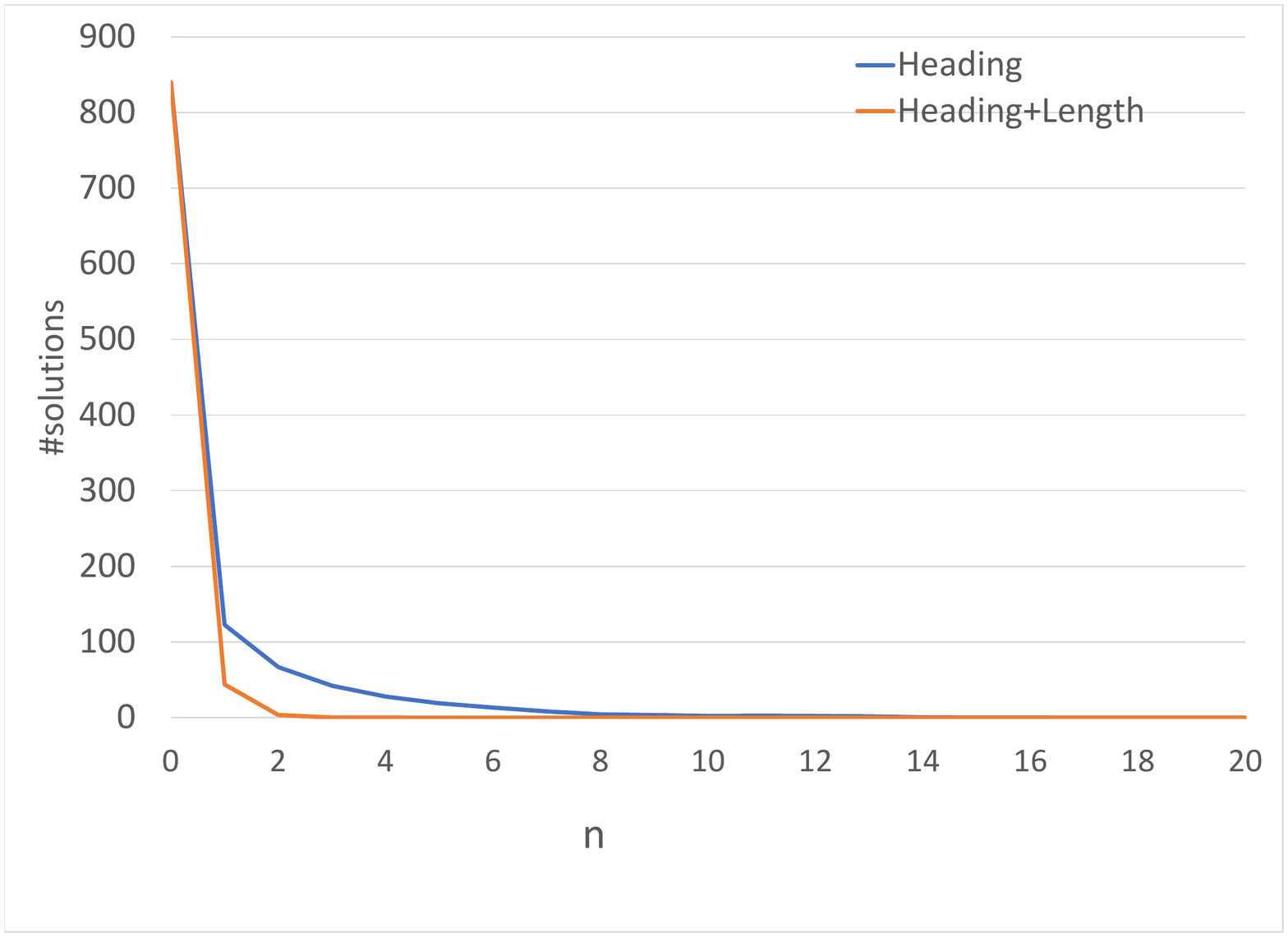}\label{fig:obs_sol}}
	\subfigure[]{\includegraphics[width=0.47\linewidth, viewport=64 90 716 552, clip=true]{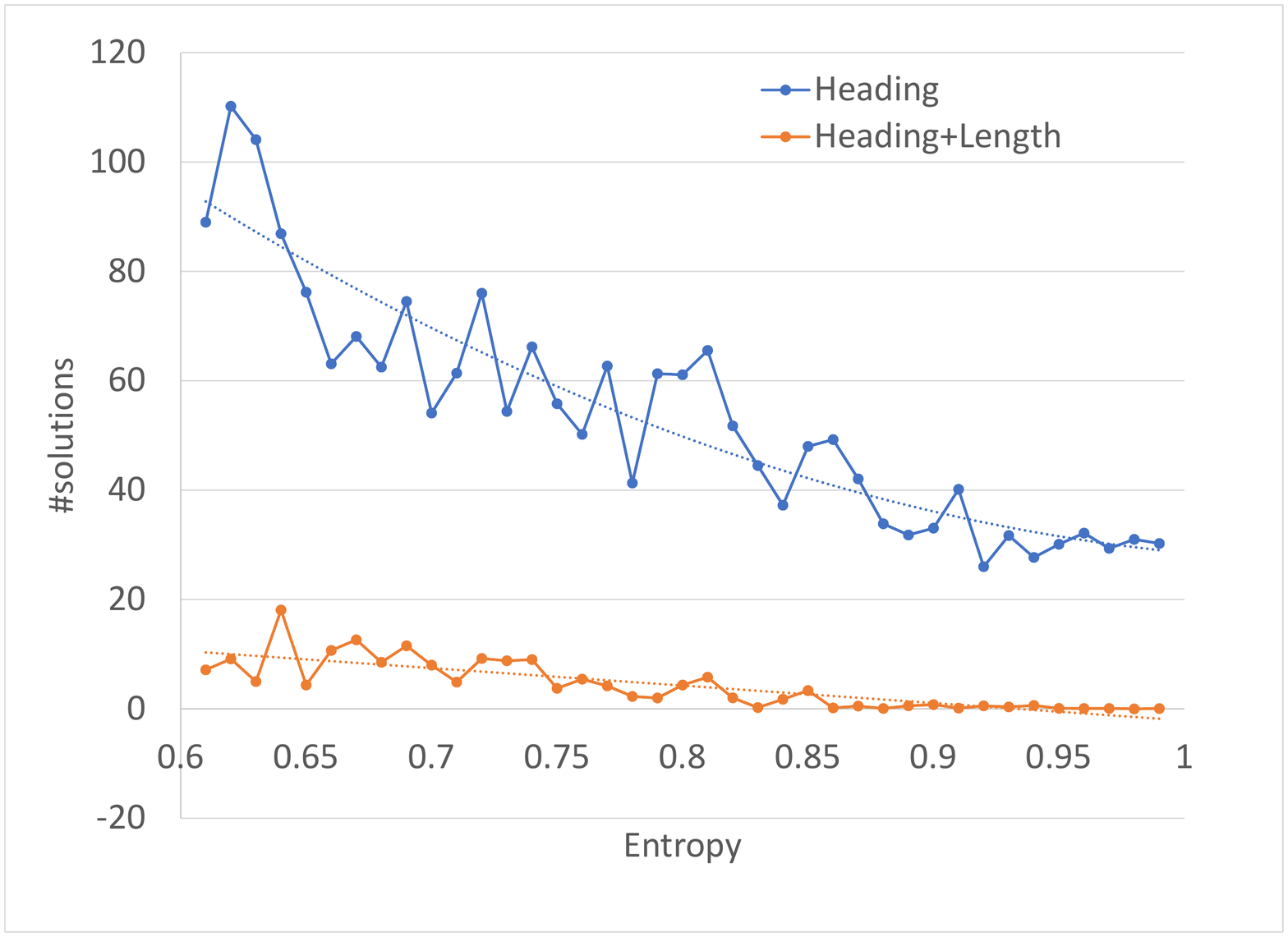}\label{fig:entro_sol}}
	\caption{ (a) Entropy of 100 cities. (b) Heading entropy distribution of 100 cities. (c) Heading and length entropy distribution of 100 cites. (d) $\#$solutions with respect to map entropy values (heading only) and $n$. (e) $\#$solutions with respect to map entropy values (heading+length) and $n$. (f) $n$ versus $\#$solutions with fixed map entropy = 0.86. (g) Map entropy values versus $\#$solutions with $n = 3$.}
	\label{fig:MapEntropyAnaysis}
\end{figure}

\subsubsection{Physical Experiments}\label{ssec:global_exp}
We also compare the two aforementioned methods in physical experiments. Again, the speed is described in $n$ needed to reach a unique solution. Smaller $n$ is more desirable. 
We test three sequences from CSData on CSMap, two sequences on KITTI00Map and two sequences on KITTI05Map. The comparison results are shown in the last two columns of Tab.~\ref{tab:MapsInfo}. 
In all tests, GBPL takes $n=3.1$ in average with a standard deviation of $0.69$ to localize the vehicle while PLAM takes $n=6.3$ on average with a standard deviation of $2.29$ in comparison. As expected, GBPL has a faster localization speed than that of PLAM. As shown in Tab.~\ref{tab:MapsInfo}, the entropy values (heading+length) of CSMap, KITTI00Map and KITTI05Map are 0.724, 0.877, and 0.797, respectively. By checking the results in Fig.~\ref{fig:Entropy_exp_hl}, $n$ required for reaching a unique solution in the real map agrees with simulation results.

\subsection{Localization Alignment and Verification Test}
Global localization only provides an initial position and the accuracy of continuous localization is determined by the LAV thread. We show localization accuracy result for all seven test sequences. PLAM does not have the capability of continuous localization and hence is not tested here. We only compare GBPL result with the ground truth.

\subsubsection{Ground Truth and Evaluation Metric} The ground truth in our experiments is the actual GPS trajectory. The localization error is defined as the Euclidean
distance between the estimated aligned trajectory and the ground truth. The localization errors are measured in meters.

\begin{figure}[ht!]
	\subfigure[]{\includegraphics[width=0.48\linewidth, viewport={119 239 477 534}, clip=true]{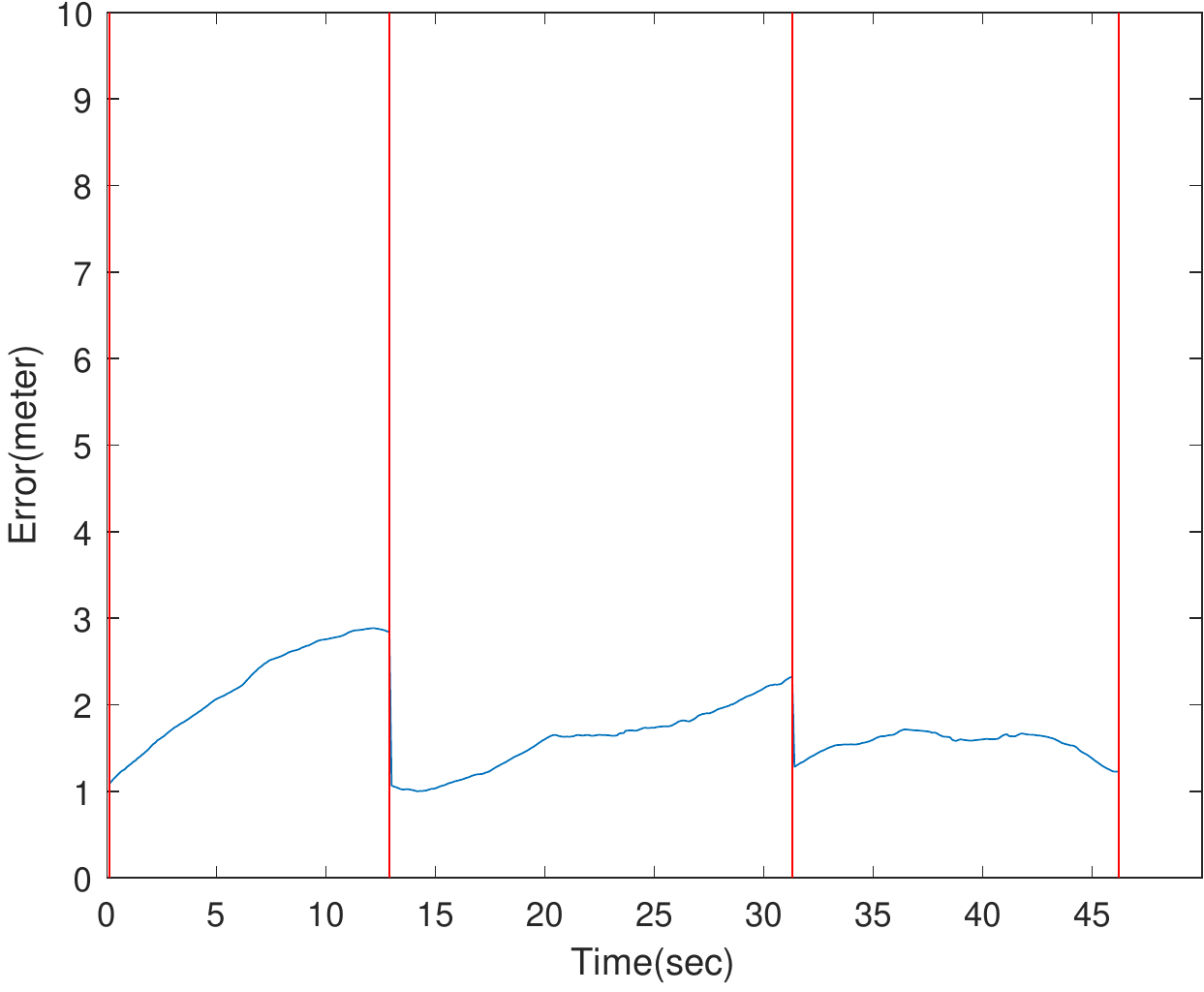}}
	\subfigure[]{\includegraphics[width=0.48\linewidth, viewport={119 239 477 534}, clip=true]{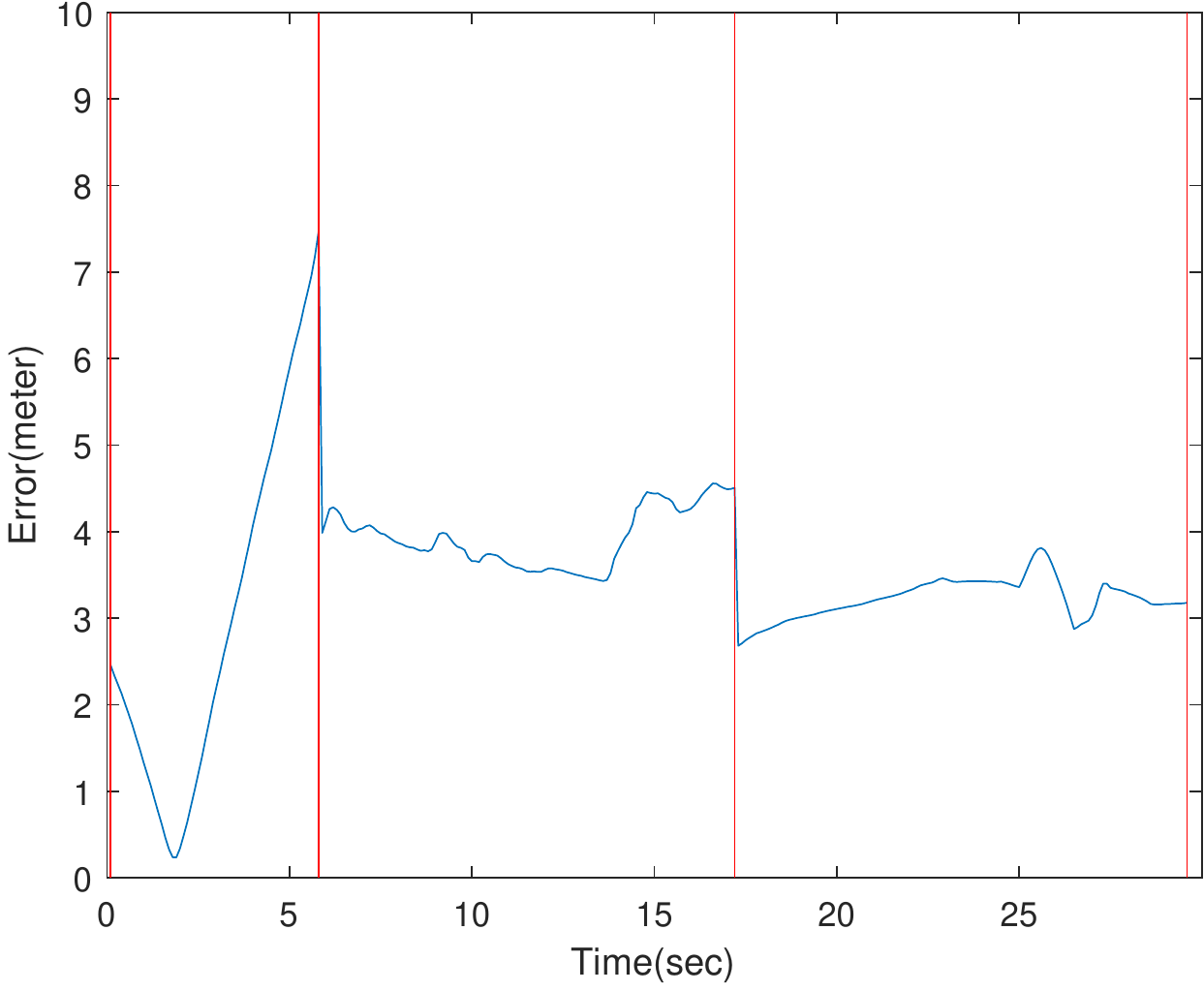}}
	\subfigure[]{\includegraphics[width=0.48\linewidth, viewport={119 239 477 534}, clip=true]{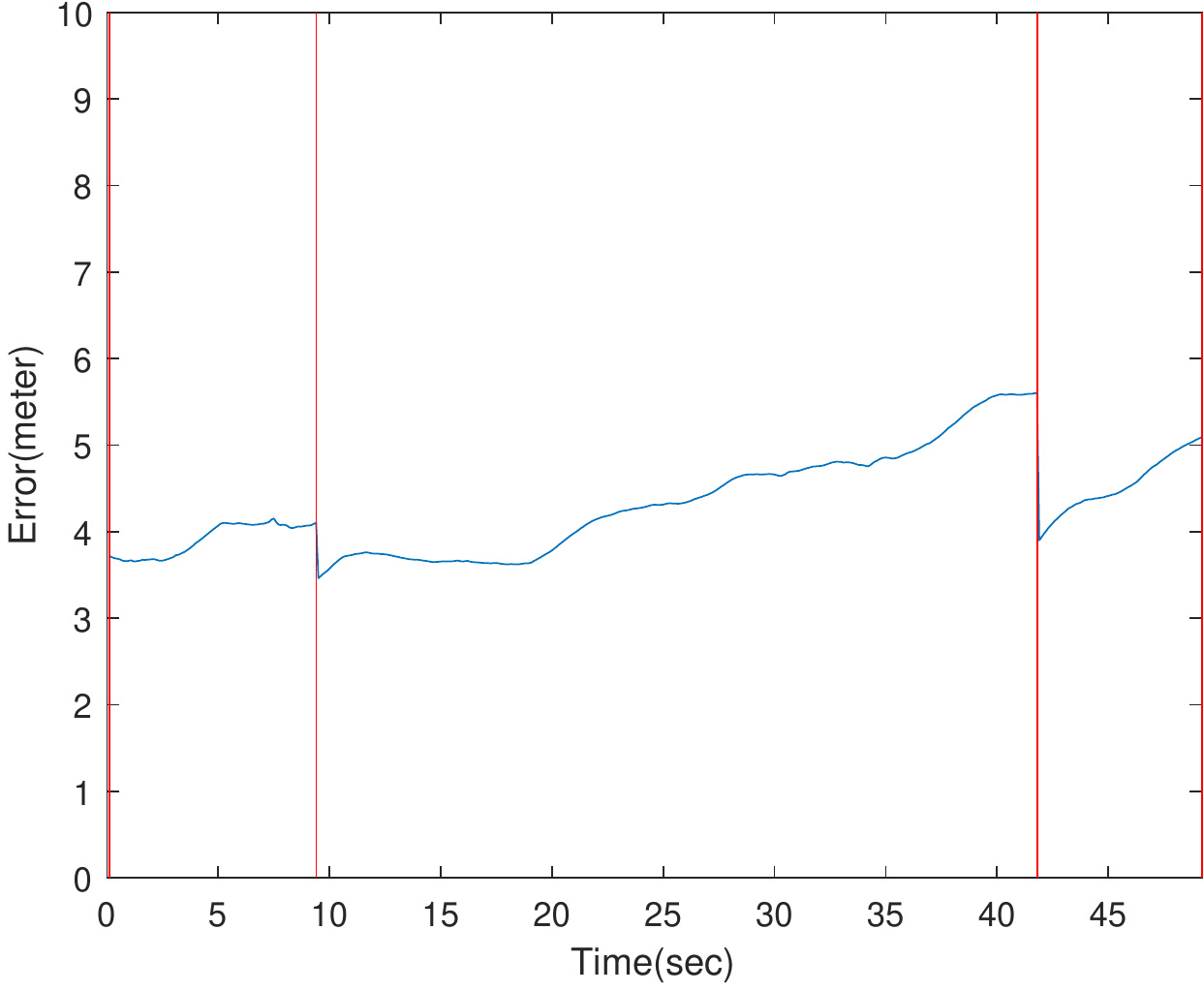}}
	\subfigure[]{\includegraphics[width=0.48\linewidth, viewport={119 239 477 534}, clip=true]{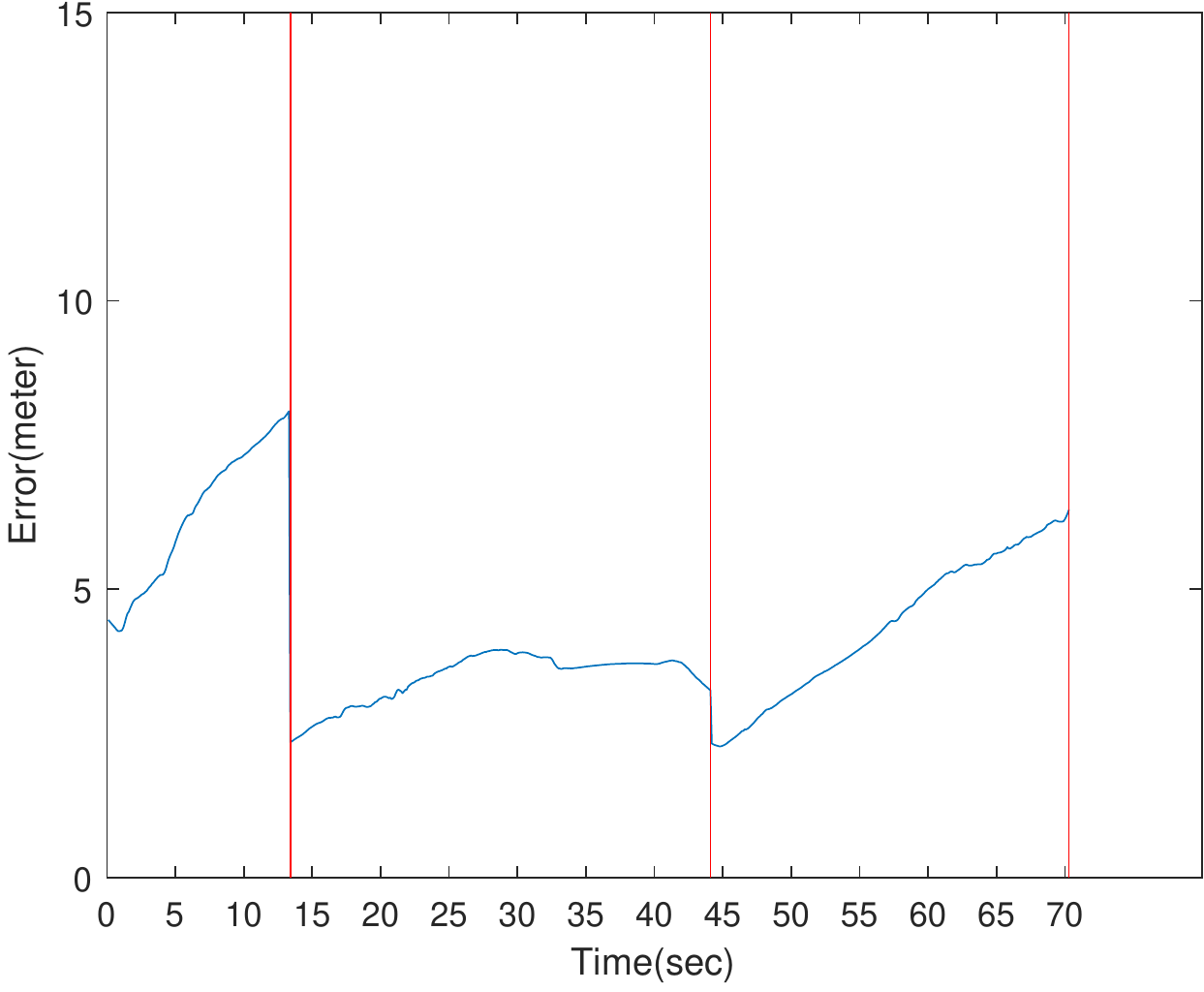}}
	\caption{LAV accuracy results using KITTI sequences on KITTI00Map and KITTI05Map: (a)  KITTI00-1, (b)  KITTI00-2, (c)  KITTI05-1, and (d)  KITTI05-2.}\label{fig:loc_error_kitti}
\end{figure}
\begin{figure}[ht!]
	\subfigure[]{\includegraphics[width=0.48\linewidth, viewport={119 239 477 534}, clip=true]{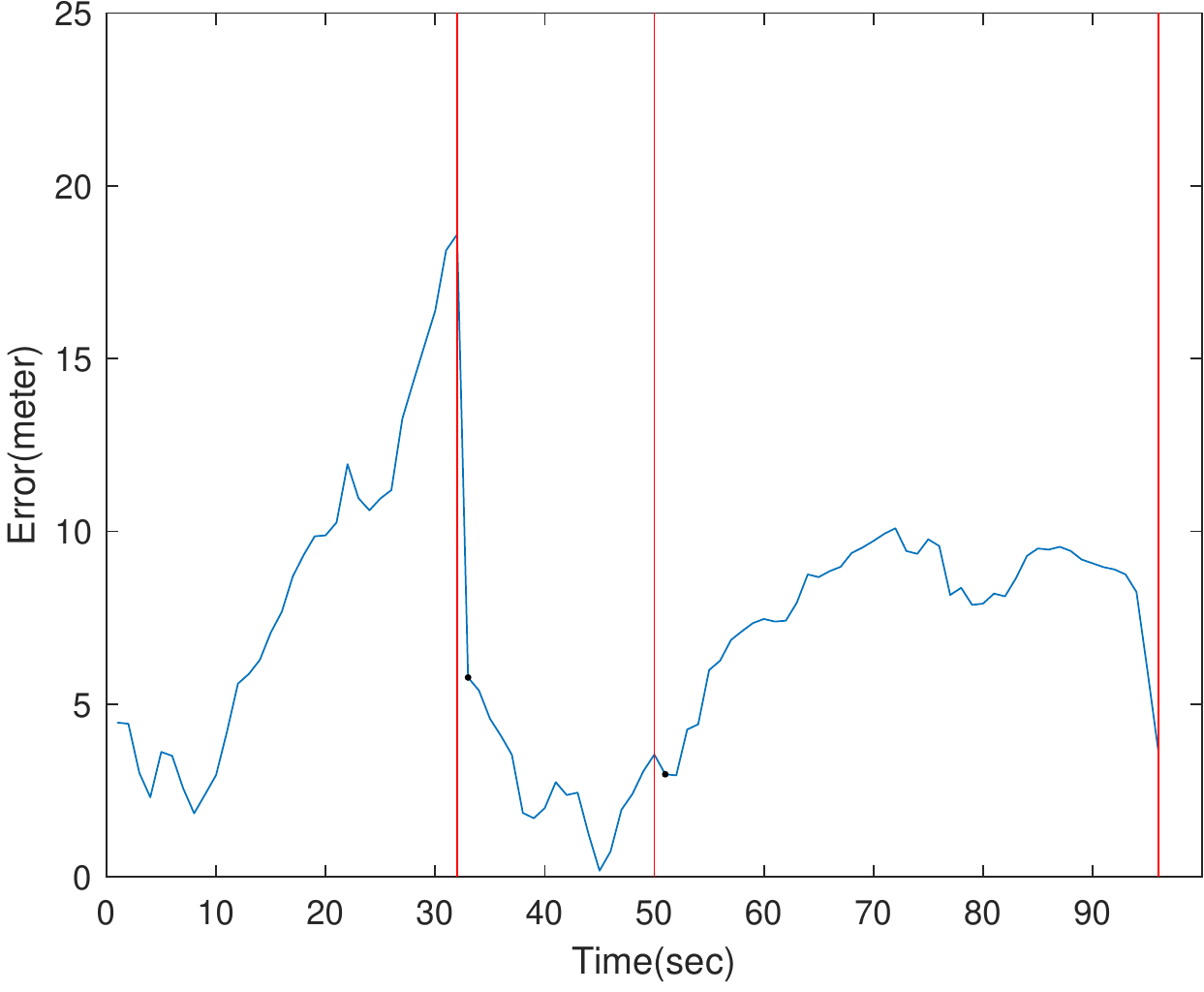}}
	\subfigure[]{\includegraphics[width=0.48\linewidth, viewport={119 239 477 534}, clip=true]{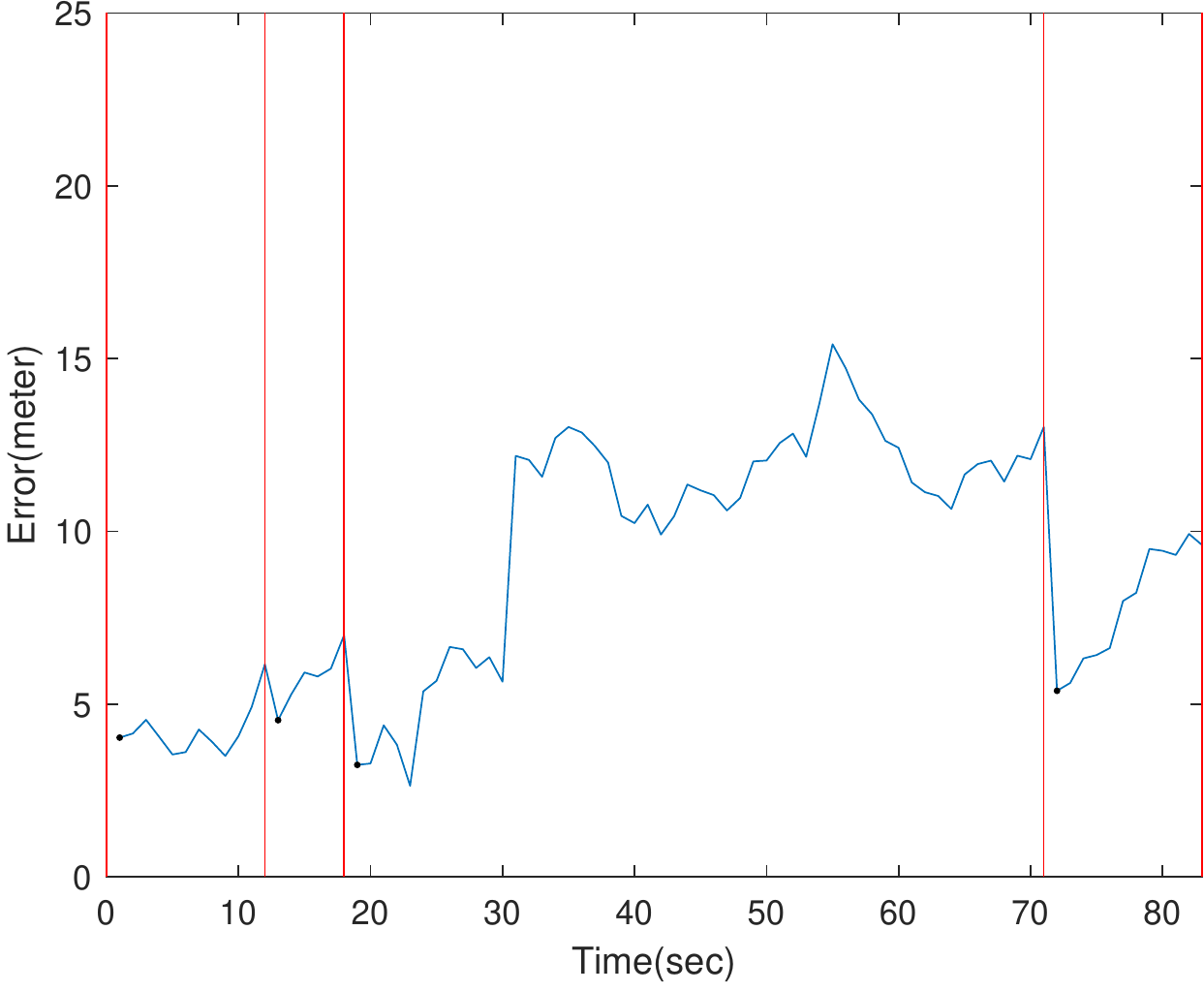}}
	\subfigure[]{\includegraphics[width=0.48\linewidth, viewport={119 239 477 534}, clip=true]{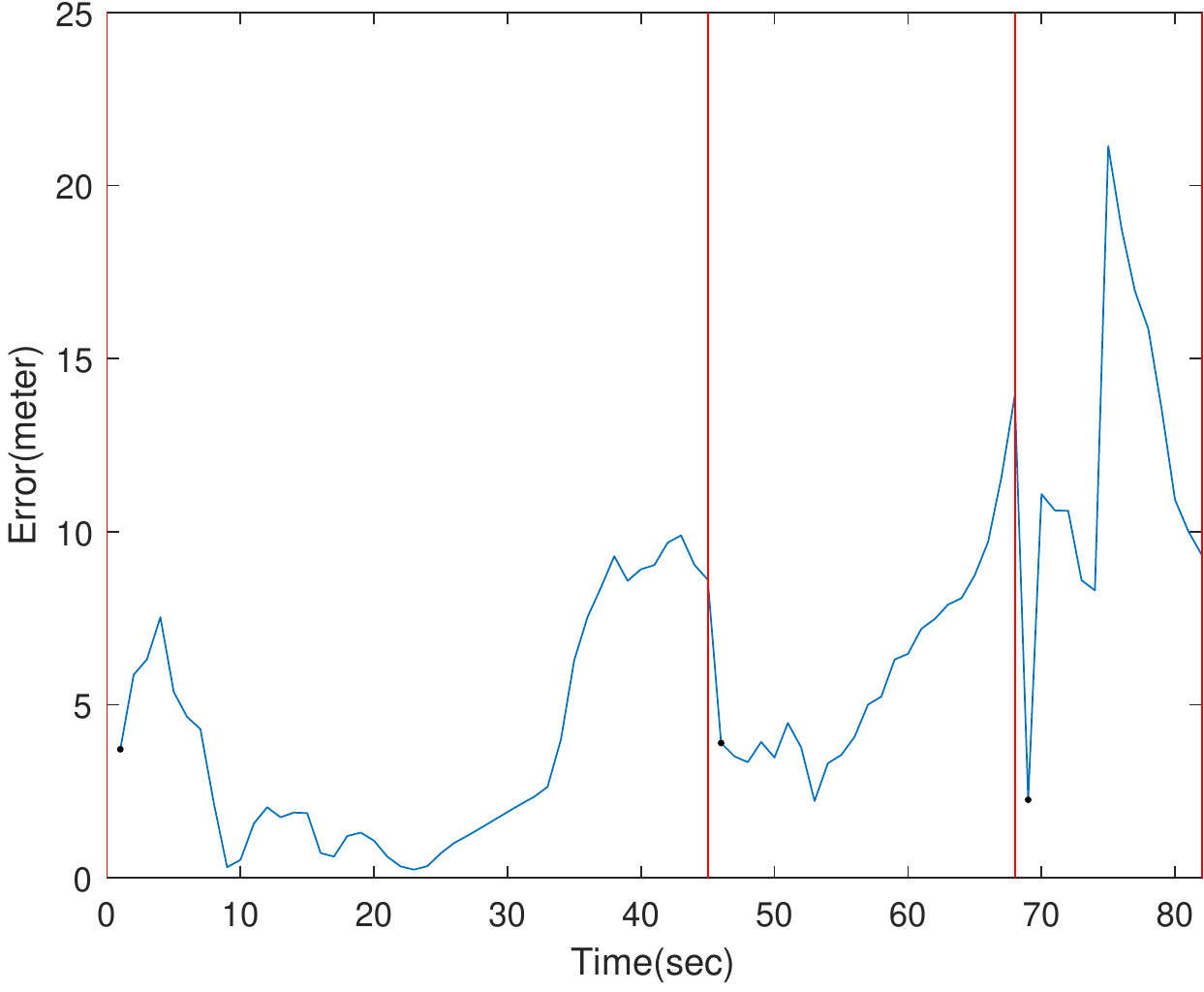}}
	\caption{LAV accuracy results using CSData on CSMap: (a) CS-1, (b)  CS-2, and (c) CS-3.}\label{fig:loc_error_BCS}
\end{figure}
\subsubsection{Accuracy Results} Figs.~\ref{fig:loc_error_kitti} and \ref{fig:loc_error_BCS} show the accuracy results by plotting the localization errors of each sequence. Red vertical lines are where LAV is excuted, i.e., when turns are detected. The first red vertical line corresponds to where we obtain global location. In all test sequences, the error in vehicle position is reduced to less than $5m$ when LAV runs at the moments indicated by the red lines. After that error slowly grows until reaching the next LAV moment.  This matches the expected map uncertainty (around $10m$). The localization accuracy of CSData on CSMap appears to be less than that of KITTI data. This is mostly due to the fact that the ground truth of CSData is not as accurate as that of the KITTI dataset. CSData uses the GPS receiver on the cell phone with an accuracy of about 10 meters or worse while the GPS receiver for KITTI data set is high quality GPS (model RT3000v3) with an accuracy of 1 centimeter.

\begin{figure}[ht!]
	\subfigure[]{\includegraphics[width=0.49\linewidth, viewport=104 238 478 534, clip=true]{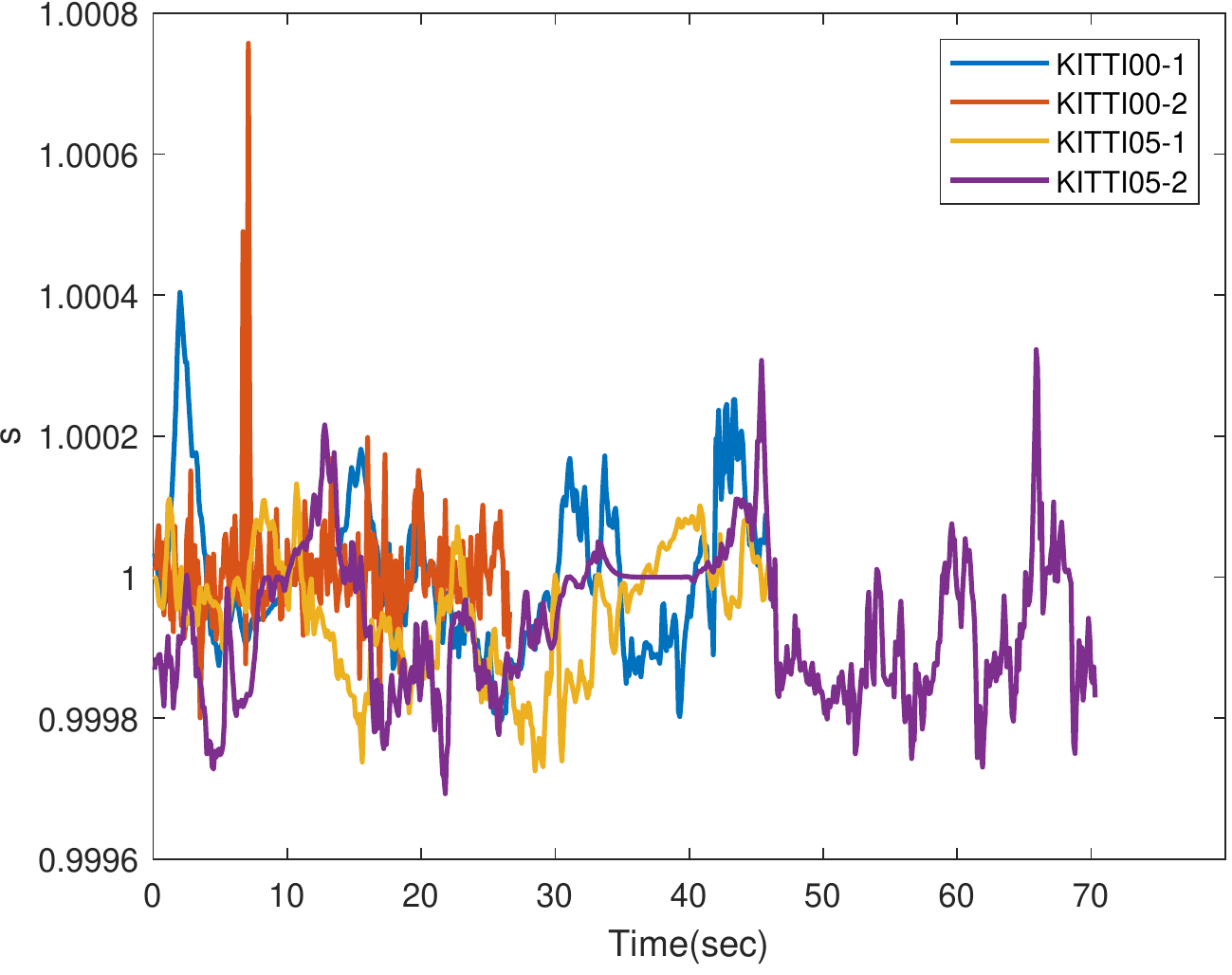}}
	\subfigure[]{\includegraphics[width=0.49\linewidth, viewport=104 238 478 534, clip=true]{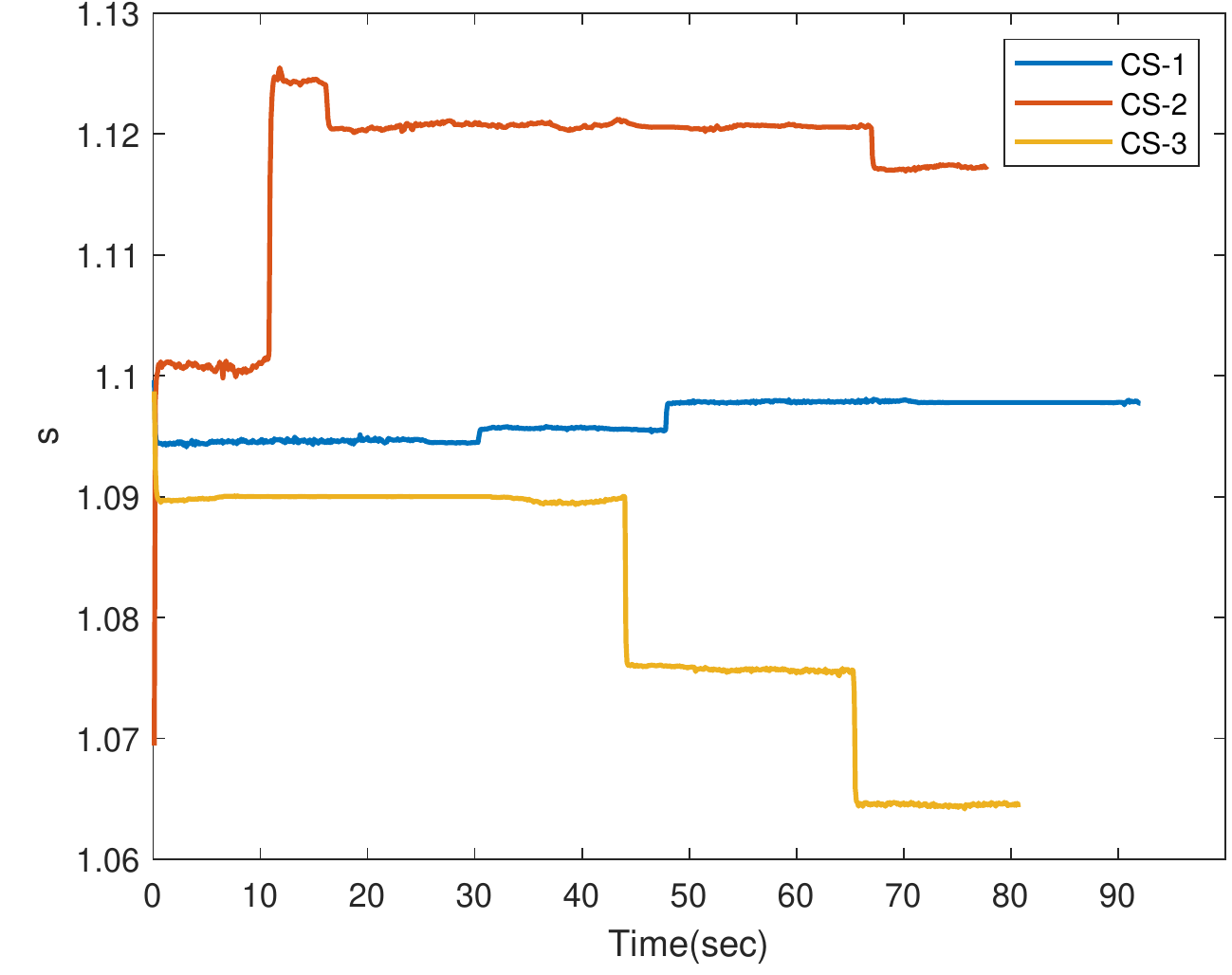}}
	\caption{Scale and slip factor value over time in EKF \eqref{eq:obs_scale}: (a) KITTI data and (b) CSData. Note the sequences are color coded and are not of the same length in time.}
	\label{fig:loc_scale}
\end{figure}
\begin{figure}[ht!]
	\subfigure[]{\includegraphics[width=0.49\linewidth, viewport=104 238 478 545, clip=true]{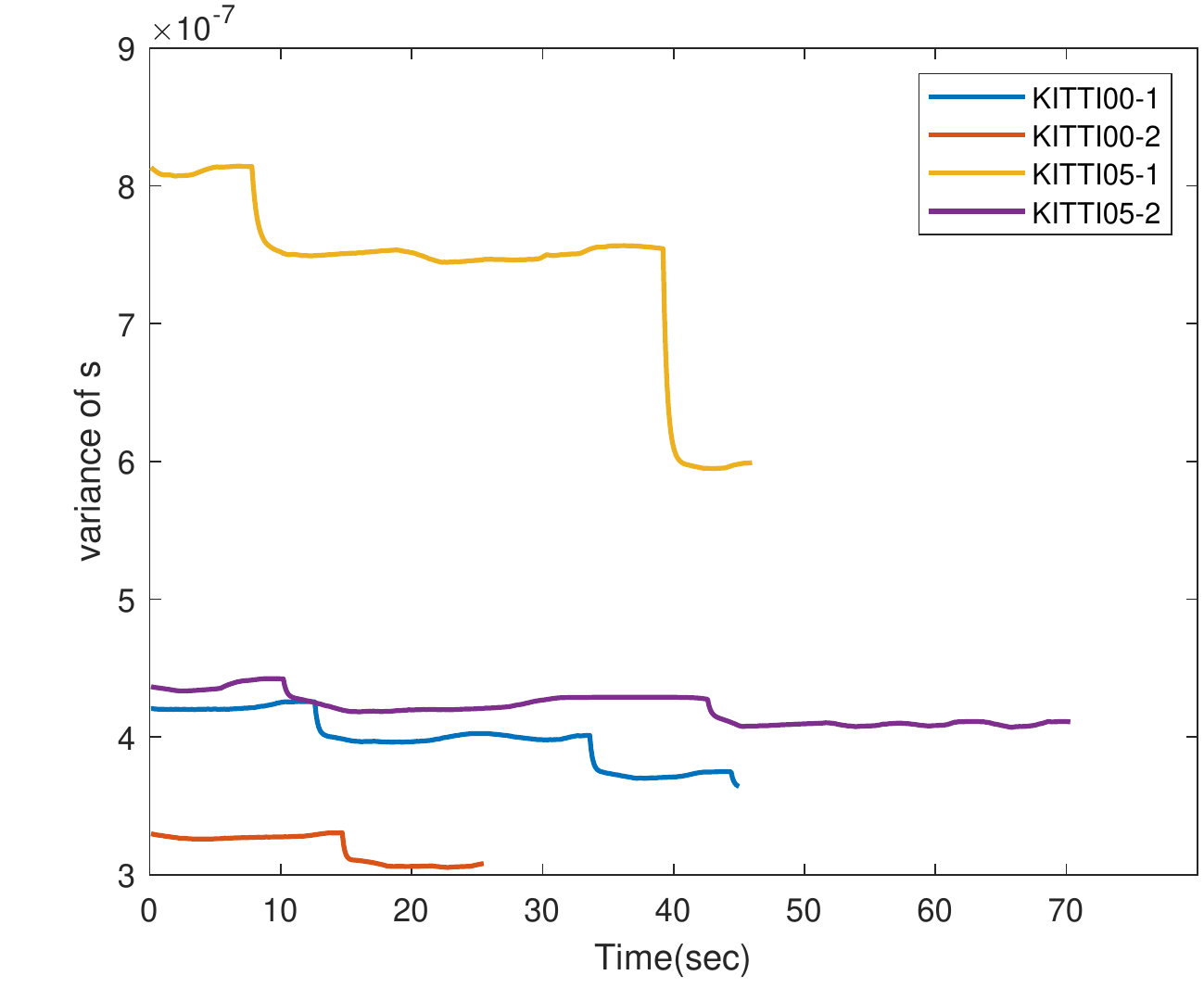}}
	\subfigure[]{\includegraphics[width=0.49\linewidth, viewport=104 238 478 545, clip=true]{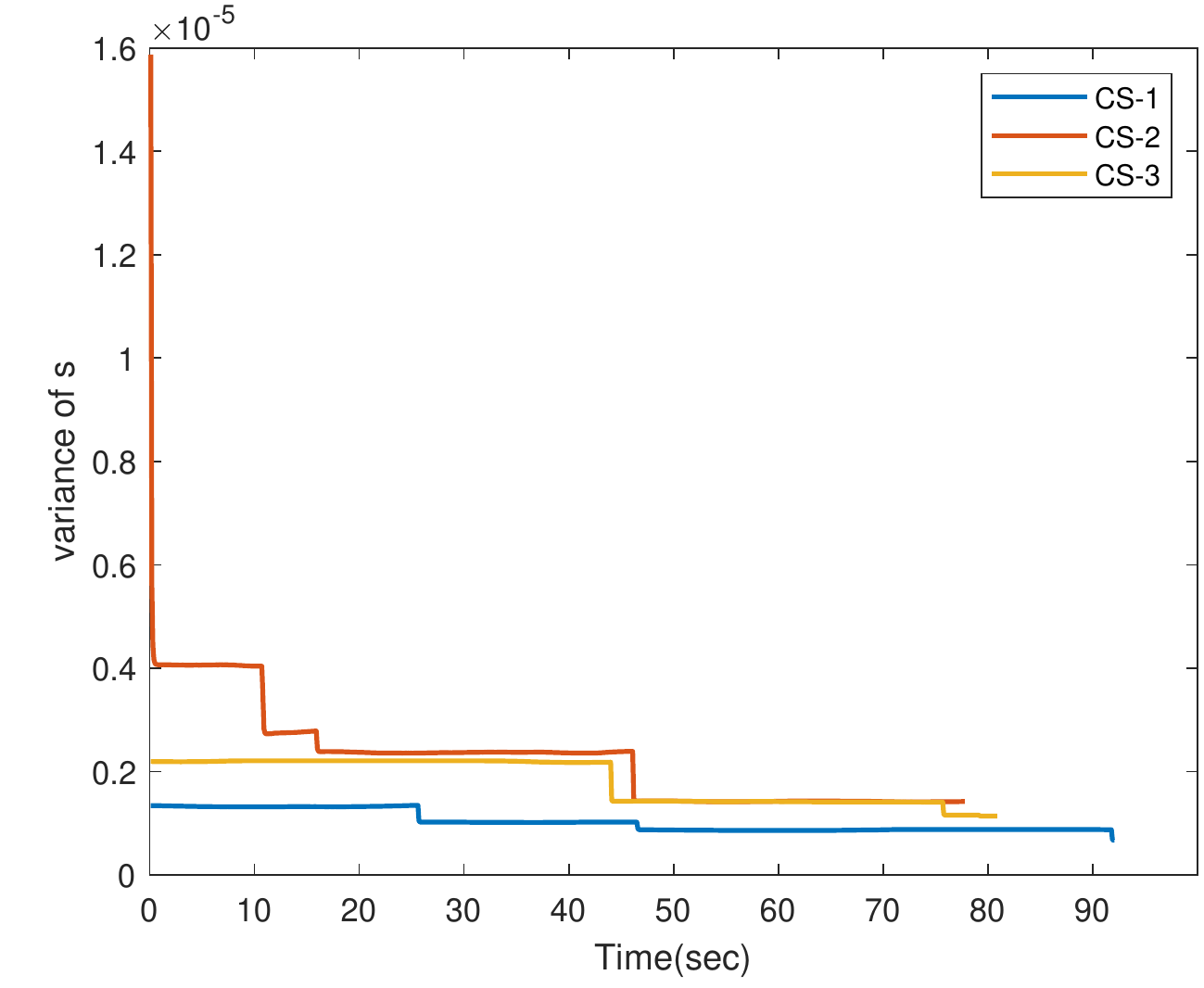}}
	\caption{Scale and slip factor variance over time in EKF: (a) KITTI data and (b) CSData. Note the sequences are color coded and are not of the same length in time.}
	\label{fig:loc_scalevar}
\end{figure}
\subsubsection{Scale and Slip Factor} Fig.~\ref{fig:loc_scale} shows the estimated SSF in EKF (i.e. $s_j$ in \eqref{eq:obs_scale}). These results show the effectiveness of LAV in detecting systematic bias in wheel odometry. For CSData, SSF values are between 1.09 to 1.15 while the SSF values from KITTI data are close to 1.00. It is clear that the vehicle velocity from the Panda OBD II dongle contains bias. It tends to underestimated vehicle velocity by about 10\%. This may be due to incorrect parameters in gear ratio or wheel/tire size. Also, the fluctuation in SSF in CSData is also large. This may also be a result of less accurate GPS values or variable tire inflation status since data is collected at different times over several months. Nonrigid mounting of the cellphone also contributes to the issue. Nevertheless, our GBPL algorithm is robust to these factors and still provides a good localization result.
%%　scale variance
We also shows the variance of $s_j$ in Fig.~\ref{fig:loc_scalevar}. These results show $\sigma^2_{s_ssf}$ decreasing as travel length increases as in Lemma~\eqref{lem:scale}.  

%We also summarize the quantitative results Tab.~\ref{tab:exp_loc_refine}.

\begin{comment}
\begin{table}[b!]
	\caption{Results of Localization Alignment Test} 
	\label{tab:exp_loc_refine}
	\resizebox{\textwidth}{!}{ 
		\begin{tabular}{cccccccc}
			\hline
			Sequence   &CS-1 &CS-2 &CS-3 &KITTI00-1 &KITTI00-2 &KITTI05-1 &KITTI05-2  \\ 
			\hline
			Average Error (m) &7.1  &8.8  &5.8  &1.8  &3.5  &4.3  & 4.3\\
			Max Error (m)     &18.6 &15.4 &21.1 &2.9  &7.5  &5.6  &8.1\\
			Average Scale     &1.09 &1.09 &1.15 &1.00 &1.00 &1.00 &1.00\\ 
			\hline                                       
		\end{tabular}
	}
\end{table}
\end{comment}

%%%%%%%%%%%%%%%%%%%%%%%%%%%%%%%%%%%%%%%%%%%%%%%%%%%%%%%%%%%%%%%%%%%%%%%%%%%%%%%%
\section{Conclusion and Future Work} \label{sec:conclusion}
We reported our GBPL method that did not rely on the perception and recognition of external landmarks to localize robots/vehicles in urban environments. The proposed method
is designed to be a fallback solution when everything else fails due to poor lighting conditions or bad weather conditions. The method estimated a rudimentry vehicle trajectory computed from an IMU, a compass, and a wheel encoder and matched it with a prior road map. To address the drifting issue in the dead-reckoning process and the fact that the vehicle trajectory may not overlap with road waypoints on the map, we developed a feature-based Bayesian graph matching where features are long and straight road segments. GBPL pre-processed maps into an HLG which stores all long and straight segments of road as nodes to facilitate global localization process. Once the map matching is successful, our algorithm tracks vehicle movement and use the map information to regulate EKF's drifting issue. The algorithm was tested in both simulation and physical experiments and results are satisfying.  

In the future, we are interested in extending the work to design a multiple vehicle/robot collaborative localization scheme under \emph{ad hoc} vehicle-to-vehicle communication framework.  We will report new results in the future publications.
\section*{Acknowledgment}
We would like to thank C. Chou, B. Li, S. Yeh, A. Kingery, A. Angert,  D. Wang, and S. Xie for their input and contributions to the NetBot Lab at Texas A\&M University.

%%%%%%%%%%%%%%%%%%%%%%%%%%%%%%%%%%%%%%%%%%%%%%%%%%%%%%%%%%%%%%%%%%%%%%%%%%%%%%%%
%%% reference(1)

\bibliographystyle{IEEEtran}
\bibliography{./bibs/jas_gyroloc,./bibs/jas,./bibs/Yan,./bibs/dez}
\end{document}